\documentclass[twoside,11pt]{article} 
\usepackage{jmlr2e} 
\usepackage{algpseudocode, algorithm}
\usepackage{bm}
\usepackage{amsmath}
\usepackage{amsfonts}
\usepackage{amssymb}
\usepackage{graphicx}
\usepackage{color}
\usepackage{empheq}
\usepackage{url}
\usepackage{graphicx}
\usepackage{caption}
\usepackage{subcaption}

\newcommand{\tab}{\hspace*{2em}}

\firstpageno{1}

\begin{document}

\title{Nested Hierarchical Dirichlet Processes for Multi-Level Non-Parametric Admixture Modeling}

\author{\name Lavanya Sita Tekumalla\thanks{The first two authors have contributed equally to the paper} \email lavanya.iisc@gmail.com \\
       \addr Dept of Computer Science and Automation\\
       Indian Institute Of Science\\
       Bangalore, India\\
       \AND
       \name Priyanka Agrawal\footnotemark[1] \email priyanka.svnit@gmail.com \\
       \addr IBM Research Labs\\       
       Bangalore, India\\
       \AND
       \name Indrajit Bhattacharya \email indrajitb@gmail.com \\
       \addr IBM Research Labs\\       
       Bangalore, India}

%\author{}       
 \editor{}	

\maketitle

\begin{abstract}
%Many applications require discovery of entities and their associated topics from large document collections.
%Typically, specifying the number of entities and topics ahead of time is difficult at best.
%This naturally calls for non-parametric modeling, where the the number of entities and topics are inferred from the data.
%The existing family of author-topic models fall short in this respect. 
%At the other end, the existing models in the Dirichlet Process family do not lend themselves to capturing entities as mixtures over topics and documents as mixtures over entities in a doubly non-parametric setting.
%In this work, we propose a model that is non-parametric in both entities and topics by using the nested HDO, which is a two layer coupling of HDPs, one for entities and another for topics.
%We provide a restaurant family interpretation of this model using a corporate cafeteria process, which allows us to design a collapsed Gibbs Sampling based inference algorithm for our model.
%In our experiments, ...   

Dirichlet Process(DP) is a Bayesian non-parametric prior  for infinite mixture modeling, 
where the number of mixture components grows with the number of data items.
The Hierarchical Dirichlet Process (HDP), often used for non-parametric topic modeling, 
is an extension of DP for grouped data, 
where each group is a mixture over shared mixture densities. The Nested Dirichlet Process (nDP), 
on the other hand, is an extension of the DP for learning group level distributions 
from data, simultaneously clustering the groups. It allows group level distributions to be shared 
across groups in a non-parametric setting, leading to a non-parametric mixture of mixtures.
The nCRF extends the nDP for multi-level non-parametric mixture modeling, enabling modeling topic hierarchies.
However, the nDP and nCRF do not allow sharing of distributions as required in many applications, motivating the need for 
multi-level non-parametric \textit{admixture} modeling. We address this gap by proposing multi-level nested 
HDPs (nHDP) where the base distribution of the HDP is itself a HDP at each level thereby leading to
admixtures of admixtures at each level.

We motivate the need for nHDP by applying a two-level version of it for non-parametric entity topic modeling,
where an inner HDP creates a countably infinite set of topic mixtures and associates them with entities, while an outer HDP 
associates documents with these entities or topic mixtures. Making use of a multi-level nested Chinese Restaurant Franchise 
(nCRF) representation for the nested HDP, we propose a collapsed Gibbs sampling based inference algorithm for 
the model. Because of couplings between various HDP levels, scaling up is naturally a challenge for the inference 
algorithm. We propose a scalable inference algorithm by extending the direct sampling scheme of the HDP to multiple levels. 
In our experiments for non-parametric entity topic modeling on two real world research corpora, we show that, 
even when large fractions of author entities are hidden, the nHDP is able to generalize significantly better than 
existing models. More importantly, using  nHDP, we are able to detect missing authors at a reasonable level of accuracy.

%The Nested Dirichlet Process (NDP) builds on the HDP to 
%cluster the documents, but allowing them only from a set of specific topic mixtures. 
% However, the nDP does not allow sharing In many applications, 
% such a set of topic mixtures may be identified with the set of entities for the collection. However, in many applications, 
% multiple entities are associated with documents, and often the set of entities may also not be known completely in advance. 
% In this paper, we address this problem using a nested HDP (nHDP), where the base distribution of an outer HDP is itself an HDP. 
% The inner HDP creates a countably infinite set of topic mixtures and associates them with entities, while the outer HDP 
% associates documents with these entities or topic mixtures. Making use of a nested Chinese Restaurant Franchise (nCRF) 
% representation for the nested HDP, we propose a collapsed Gibbs sampling based inference algorithm for the model. Because of 
% couplings between two HDP levels, scaling up is naturally a challenge for the inference algorithm. We propose a scalable inference 
% algorithm by extending the direct sampling scheme of the HDP to two levels. In our experiments on two real world research corpora, 
% we show that, even when large fractions of author entities are hidden, the nHDP is able to generalize significantly better than 
% existing models. More importantly, we are able to detect missing authors at a reasonable level of accuracy. 

\end{abstract}

\section{Introduction}\label{sec:intro}
Dirichlet Process mixture models [\cite{antoniak:aos74}] allow for non-parametric or infinite mixture modeling, where the number 
of densities or mixture components is not fixed ahead of time, but is allowed to grow (slowly) with the number of data items.
This is achieved by using as a prior the Dirichlet Process (DP), which is a distribution over distributions, and has the additional 
property that draws from it are discrete (w.p. 1) with infinite support [\cite{antoniak:aos74,ferguson:aos73}].
The popular LDA model [\cite{LDA}] may be considered as a parametric restriction of the HDP mixture model.
LDA and its non-parametric counterpart HDP have since been used extensively as a prior for modeling of text collections 
[ \cite{hdpapp1, hdpapp2}].
%The Chinese Restaurant Process (CRP) analogy, which gives the conditional distributions for successive draws after integrating out the DP, enables straight-forward Gibbs Sampling algorithms for inference using DP mixture models, and has helped in popularizing their application in various clustering problems.
However, many applications require joint analysis of groups of data, such as a collection of text documents, where the mixture
components, or topics (as they are called for text data), are shared across the documents.
This calls for a coupling of multiple DPs, one for each document, where the base distribution is discrete, and shared. 
The hierarchical Dirichlet Process (HDP) [\cite{HDP}] does so by placing a DP prior on a shared base distribution, so that the 
model now has two levels of DPs.
%Inference using the HDP makes use of a Chinese Restaurant Franchise (CRF) analogy, obtained by integrating out the two levels of DPs.

The HDP mixture model belongs to the family of non-parametric admixture models [\cite{erosheva:pnas04}], where 
each composite data item or group gets assigned to a mixture over the mixture components or topics, enabling 
group specific mixtures to share mixture components. Hence the HDP family leads to group level distributions with share mixture 
component distributions leading to a family of distributions over distributions.
While this adds more flexibility to the groups of data items, the ability to cluster groups themselves is lost, since each group now has 
a distinct mixture of topics associated with it. This additional capability is desired in many applications. For instance, consider 
the  analysis of patient profiles in hospitals [\cite{NDP}], where we would like to cluster patients in each hospital and 
additionally cluster the hospitals with common distributions over patient profiles. This is achieved by constructing a DP mixture 
over possible group level distributions from which distribution for each hospital is drawn, thus clustering hospitals based on the
specific group level distribution chosen. This DP mixture has a base distribution that is itself a DP 
(instead of a draw from a DP, like in the case of HDP), from which the group level distributions
over patient profiles are drawn. Since the patient profiles are themselves appropriately chosen distributions, 
%the nested DP (nDP) \cite{NDP}, which has a DP corresponding to each group level distribution(over patient profiles). The DP prior over these group level distributions 
%is for each hospital is drawn from 
%a DP prior over these group level distributions which are coupled through the same base distribution, 
%which is a DP itself, 
the nDP results in a distribution over distributions over distributions, unlike the HDP and the DP, 
which are distributions over distributions. The nDP model therefore becomes a prior for non-parametrically modeling  
mixture of mixtures over appropriately chosen component distributions. The nested CRP (nCRP) [\cite{blei:jacm10}], a closely related model, 
proposes a model for \textit{multi-level} hierarchical mixture modeling to discover topic hierarchies of arbitrary depth through 
the predictive distribution obtained by integrating out the DP in a multi-level nDP.  

While the nDP family enables multi-level non-parametric mixture modeling, it is limited by the fact that it does not 
allow sharing of mixture components across group specific distributions at each level.  
For instance, in the previous example, group level distributions in hospitals do not share mixture components (patient profiles).
In several real world applications, a need arises for multi-level non-parametric mixture modeling where at each level, 
group specific mixtures are required to share mixture components. This necessitates \textit{multi-level non-parametric 
admixture modeling}. For instance, imagine a corpus containing descriptions related to entities, such as a 
shared set of researchers who have authored a large body of scientific literature, or a shared set of personalities discussed 
across news articles, such that each entity can be represented as a mixture of topics.
Here, topic mixtures, corresponding to entities, are required to be shared across data groups or documents. In addition, 
 we would like topics themselves to be shared across the topic mixtures corresponding to entities. 

One could attempt to model this problem of non-parametric entity-topic modeling with nDP. The nDP can be imagined as 
first creating a discrete set of mixtures over topics, each mixture representing an entity, and then choosing exactly 
one of these entities for each document. In this sense, the nDP is a mixture of admixtures.
However, a major shortcoming of the nDP for entity analysis is the restrictive assumption of a single entity being 
associated with a document. In research papers, multiple authors are associated with any document, and any news article 
typically discusses multiple news personalities. This requires each document to have a distribution over entities.
In other words, we need a model that is an admixture of admixtures motivating the need for multi-level admixture modeling.
%The Author-Topic Model (ATM) \cite{AT}, which models authors associated with documents, belongs to this class, 
%but is restrictive in other ways, since it requires the authors to be observed for documents, and also assumes the number of 
%topics to be known.

In this paper, we address non-parametric multi-level admixture models. 
To the best of our knowledge, there is no prior work that addresses this problem. 
%{\bf Applications ...}
We propose the nested HDP (nHDP), comprising of multiple levels of HDP, where the base distribution of 
each HDP is itself an HDP.
For inference using the nHDP, we propose the nested CRF (nCRF), which extends the Chinese Restaurant Franchise (CRF) 
analogy of the HDP to multiple levels by integrating out each HDP.
%We explore extensions of the three CRF-based HDP inference algorithms for nHDP using the nCRF. 
However, due to strong coupling between the CRF layers, inference using the nCRF poses computational challenges.
%We propose an algorithm that interleaves Metropolis-Hastings steps that propose restrictions to dish-assignments in the outer CRF with Gibbs sampling steps that work with accepted restrictions.
%We show that this algorithms is able to scale ...
We propose a scalable algorithm for inference in the multi-level setting with a direct sampling scheme, 
based on that for the HDP, where the mixture component associated with an observation is directly sampled at each level  
, based on the counts of table assignments and stick-breaking weights at each of the levels.

%of entities and topics, 
%where the number of topics is not known in advance, and additionally the set of entities for each document is either partly or completely unknown.
We apply the two-level nHDP to address the problem of non-parametric entity topic analysis for simultaneous
discovery of entities and topics from document collections. The two-level nHDP belongs to the same class of models 
as a two-level nDP, in the sense that it specifies a distribution over distributions (entities) over 
distributions (topics). However, unlike the nDP, it first creates a discrete set of entities, and models each group as a 
document specific mixture over these entities using a HDP. Similarly, it creates a discrete set of topics and models 
each entity as a distribution over these topics using another level of HDP leading to two levels of HDPs. 
Apart from addressing the novel problem of multi-level admixture modeling, to the best of our knowledge, ours is the first 
attempt at entity topic modeling that is non-parametric in both entities and topics. 
The Author Topic Model falls out as a parametric version of this model, when the entity set is observed for each document, 
and the number of topics is fixed. Using experiments over publication datasets using author entities from NIPS and DBLP, 
we show that the nHDP generalizes better under different levels of available author information.
More interestingly, the model is able to detect authors completely hidden in the entire corpus with reasonable accuracy.       

\section{Related Work}\label{sec:rw}	

%In the previous section, we have motivated the need for a nested HDP and have reviewed the related work in Bayesian Non Parametric Modeling. 
In this section, we review existing literature on Bayesian nonparametric modeling and entity-topic analysis.

%\subsection{Nonparametric Bayesian Modeling}
\vspace{0.05in}
\noindent
{\bf Bayesian Nonparametric Models: }
%The Latent Dirichlet Allocation, LDA\cite{LDA} is a Hierarchical Bayesian model popularly used for document modeling. 
%LDA models a document as a distribution over topics and topics as distributions over words. 
%As described in the previous section, HDP\cite{HDP} extends the DP\cite{antoniak:aos74,ferguson:aos73} to a document modeling scenario by creating a hierarchy of DPs for sharing of topic atoms between documents. 
We will review the Dirichlet Process (DP) [\cite{ferguson:aos73}], the Hierarchical Dirichlet Process (HDP) [\cite{HDP}] and the nested Dirichlet Process (nDP) \cite{NDP} in detail in the Sec. \ref{sec:bg}.
% The nDP\cite{NDP} as described in the previous section generates samples from a DP whose base distribution itself is a DP, thus simultaneously learning distributions on observations and distributions on these distributions. 

The MLC-HDP [\cite{wulsin:icml12}] is a $3$-layer model proposed for human brain seizures data.
The $2$-level truncation of the model is closely related to the HDP and the nDP.
Like the HDP, it shares mixture components across groups (documents) and assigns individual data points to the same set of mixtures, and like the nDP it clusters each of the groups or documents using a higher level mixture.
In other words, this is a nonparametric mixture of admixtures, while our proposed nested HDP is a nonparametric admixture of admixtures. 

The nested Chinese Restaurant Process (nCRP) [\cite{blei:jacm10}] extends the Chinese Restaurant Process analogy of the Dirichlet Process to an infinitely-branched tree 
structure over restaurants to define a distribution over finite length paths of trees.
This can be used as a prior to learn hierarchical topics from documents, where each topic corresponds to a node of this tree, 
and each document is generated by a random path over these topics. The nCRP is also closely connected to the nDP in that 
the predictive distribution obtained by integrating out the DPs at each level from a K-level nDP leads to an nCRP. 
However, while the nCRP and the nDP facilitate multi-level non-parametric mixture modeling, they are not suitable for modeling multi-level 
non-parametric admixtures.
%The nested Chinese Restaurant Process (nCRP) \cite{blei:jacm10} is a related model, that aims to place a non-parametric prior on the topic hierarchy in the context of 
%document modeling, inferring a document specific hierarchy that best explains the data. 
%Using the restaurant analogy, it creates the infinite depth hierarchy by associating every table in the 
%infinite table restaurant with a new restaurant, in which every table is in turn associated with the next level restaurant, and so on. 
% This model can be obtained from an infinite level nDP by integrating out the DP at every level.

An extension to the nCRP model, also called the nested HDP, has recently been proposed on Arvix [\cite{paisley:arxiv12}]. 
In the spirit of the HDP, which has a top level DP and providing base distributions for document specific DPs, 
this model has a top level nCRP, which becomes the base distribution for document specific nCRPs.
%where a group(document) specific nCRP is used in addition to a global nCRP, leading to word specific clustering of paths instead of document specific clustering. 
%The authors refer to this model as the nHDP to indicate a {\em hierarchy of nDP models, each with arbitrary level of nesting}. 
In contrast, our model for multi-level non-parametric admixtures has nested HDPs, in the sense that one HDP directly serves as the base distribution for another HDP, 
like in the nested DP [\cite{NDP}], where one DP serves as the base distribution for another DP.
% This is different from our model, which is a nested HDP, with two level nesting, where the base distribution of our entity level outer HDP is itself a HDP over topics.
This parallel with the nested DP motivates the nomenclature of our model as the nested HDP.

%There are several models that incorporate the knowledge of  observed author entities associated with each document in learning the topics.
Next, we briefly review prior work on entity-topic modeling, that involves simultaneously modeling entities and topics 
in documents, an application we use throughout the paper to motivate our model.
%While there has been considerable work on entity-topic analysis, most of the existing work is focused on parametric models that assume that the entities of a document are completely observed.
The literature mostly contains parametric models, where the number of topics and entities are known ahead of time.
The LDA model [\cite{LDA}] is the most popular parametric topic model, that infers a known number of latent topics
from document collections. The LDA models the document as a distribution over a finite set of topics and the topics as 
distribution over words.
The author-topic model (ATM) [\cite{AT}] extends the LDA to capture {\it known} authors of each document by modeling a 
document as a unifom distribution over a known author set and authors as distributions over topics, which are
themselves distribution over words. Hence, the ATM can be used for parametric entity-topic modeling where the 
authors correspond to entities in documents.
% The author topic model ATM\cite{AT}, models the author entities as having  a distribution over topics and the document as a uniform distribution over the observed authors. 
%The Author Topic Time model\cite{ATT}, in addition to incorporating the author entities, models topic dynamics in a collaborative environment by considering the observed time-stamp associated with each document in the generative process.
The Author Recipient Topic model [\cite{ART}] distinguishes between sender and recipient entities and learns the topics and topic distributions of sender-recipient pairs. 
In [\cite{ETM}], the authors analyze entity-topic relationships from textual data containing entity words and topic words, which are pre-annotated. 
The Entity Topic Model [\cite{ETM2}] proposes a generative model which is parametric in both entities and topics and assumes observed entities for each document.

There has been very little work on nonparametric entity-topic modeling, which would enable discovery of entities in settings where entities are partially or completely unobserved in documents.  
% Further we are not aware of any work that addresses this problem when the entities for each document are only partially observed. 
The Author Disambiguation Model, [\cite{AD}] is a nonparametric model for the author entities along with topics. 
Primarily focusing on author disambiguation from noisy mentions of author names in documents, this model treats 
entities and topics symmetrically, generating entity-topic pairs from a DP prior.
Contrary to this approach, our model is capable of treating the entity as a distribution over topics, 
thus explicitly modeling the fact that authors of documents have preferences over specific topics. 
We perform experiments  in section \ref{experiments} to demonstrate the effectiveness of our model for non-parametric entity topic analysis.
% Our model is capable of on inferring the topic preferences of entities and the entities themselves, starting from a completely unobserved setting or a partially observed (incomplete) entity list.

\section{Background}\label{sec:bg}
Consider a setting where observations are organized in groups. 
%and are exchangeable within each \textit{group} and across groups. 
Let $x_{ji}$ denote the $i$-th observation in $j$-th group. 
For a corpus of documents, $x_{ji}$ is the $i$-th word occurrence in the $j$-th document.
In the context of this paper, we will use group synonymously with document, data item with word in a document.
We assume that each $x_{ji}$ is independently drawn from a mixture model and has a mixture component parameterized 
by a \textit{factor}, say $\theta_{ji}$, representing a topic, associated with it. We let these factors themselves 
be drawn independantly from a distribution $\bar H$.
For each group $j$, let the associated factors $\bm{\theta}_j = (\theta_{j1}, \theta_{j2}, \hdots)$ have a prior 
distribution $G_j$. Finally, let $F(\theta_{ji})$ denote the distribution of $x_{ji}$ given factor $\theta_{ji}$. 
Therefore, the generative model is given  by
%\begin{alignat}{2}
%\label{eq:mixture}
%&\theta_{ji} |G_j \sim G_j ~~~~~~~~~ &&\textnormal{for each } j \textnormal{ and } i, \notag \\
%&x_{ji} |\theta_{ji} \sim F(\theta_{ji}) &&\textnormal{for each } j \textnormal{ and } i.
%\end{alignat}
\begin{eqnarray}
\label{eq:mixture}
\theta_{ji} |G_j \sim G_j;\ x_{ji} |\theta_{ji} \sim F(\theta_{ji}),\ \forall j,i
\end{eqnarray}
%In this problem of placing nonparametric priors $G_j$ on the above defined mixture model, it would be natural to consider the use of Dirichlet processes. 

The central question in analyzing a corpus of documents is the parametrization of the $G_j$ distributions --- 
what parameters to share and what priors to place on them.
The LDA model [\cite{LDA}] is the most popular parametric topic model, that assumes $G_j \sim Dir(\alpha/K)$ 
is a distribution over a finite number of $k$ topics for each document. The choice of Dirichlet prior is 
based on the conjugacy of the Dirichlet distribution with the multinomial, that leads to efficient inference.
However, in most realistic scenarios, the number of topics $K$ is not known in advance.

Bayesian Non-parametric modeling, is a paradigm that enables us to choose a prior for 
$G_j$ that allows for a countably infinite number of mixture components. This enables working with mixture models
without having to fix the number of mixture components in advance by working with $G_j$ of the form 
$G_j=\sum_{k=1}^\infty \beta_k \delta_{\phi_k}$ with atoms $\phi_k \sim \bar H$, a base distribution.
%We first describe briefly Dirichlet process below and present corresponding stick-breaking construction and Chinese restaurant process representations. 
%Note that each of these representations has an analog in hierarchical Dirichlet process and our entity-topic models, described in subsequent sections.
We start with such a prior, the Dirichlet Process that considers each of the $G_j$ distributions in isolation, then the 
Hierarchical Dirichlet Process that ensures sharing of atoms among the different $G_j$s, and finally the nested 
Dirichlet Process that additionally clusters the groups by ensuring that all the $G_j$s are not distinct.

%\subsection{Dirichlet Process}
%\label{sec:dp}
\vspace{0.05in}
\noindent
{\bf Dirichlet Process: }
We start with a formal definition of the Dirichlet process as a prior for the $G_j$ distribution.
Let ($\Theta$, $\mathcal{B}$) be a measurable space.
A Dirichlet Process (DP) [\cite{ferguson:aos73,antoniak:aos74}] is a measure over measures $G_j$ on that space.
Let $\bar H$ be a finite measure on the space. Let $\alpha$ be a positive real number.
We say that $G_j$ is DP distributed with concentration parameter $\alpha$ and base distribution $\bar H$, written 
$G_j \sim $DP($\alpha, \bar H)$, if  
for any finite measurable partition $(A_1, \hdots, A_r)$ of $\Theta$, we have
\begin{equation}
\label{eqn:dp}
(G_j(A_1),  \hdots G_j(A_r)) \sim Dir(\alpha \bar H(A_1),  \hdots, \alpha \bar H(A_r)).
\end{equation}

%\subsubsection*{Stick-breaking Representation}
The {\it stick-breaking representation} provides a constructive definition for samples drawn from a DP, 
by explicitly drawing the mixture weights for $G_j$.
It can be shown [\cite{sethuraman:ss94}] that a draw $G_j$ from $DP(\alpha,\bar H)$ can be written as
\begin{eqnarray}
\label{eq:dpStick}
&\phi_k \stackrel{iid}{\sim} \bar H,\ k=1\ldots\infty;\ \ \ 
w_i\sim Beta(1,\alpha);\ \beta_i=w_i\prod_{j=1}^{i-1}(1-w_j) \nonumber \\
& G_j = \sum_{k=1}^\infty \beta_k \delta_{\phi_k},
\end{eqnarray}
where the atoms $\phi_k$ are drawn independently from $\bar H$ and the corresponding weights $\{\beta_k\}$ follow a 
stick breaking construction. This is also called the GEM distribution: $(\beta_k)_{k=1}^\infty \sim \mbox{GEM}(\alpha)$. 
The stick breaking construction shows that draws from the DP are necessarily discrete, with infinite support, and the DP 
therefore is suitable as a prior distribution on mixture components for `infinite' mixture models. 
Subsequently, $\{\theta_{ji}\}$ are drawn from $G_j$, followed by draws $\{x_{ji}\}$ (similar to Eqn. \ref{eq:mixture}).
%Let $\theta_1, \theta_2, \hdots$ be the successive random variables distributed according to $G$, where each $\theta_i$ is a factor corresponding to observation $x_i$. The likelihood is given by
%\begin{alignat}{2}
%\label{eq:dpMixture}
%&\theta_{i} |G \sim G, \notag \\
%&x_{i} |\theta_{i} \sim F(\theta_{i}).
%\end{alignat}
The generation of $G_j$ from the DP prior followed by the generation of $\{\theta_{ji}\}$ and 
$\{x_{j,i}\}$ constitutes the \textit{Dirichlet Process mixture model} [\cite{ferguson:aos73}].
%This enables assigning each word $w_{ji}$ of a document $j$ to some topic $\theta_{ji}$ from an infinite collection of topics $\{\phi_k\}$.

%In the setting of grouped data, a set of Dirichlet process can be used with one DP corresponding to each group. To introduce sharing of atoms across groups, the  base distribution $G_B$  for the group specific Dirichlet process is itself a draw from Dirichlet process. This model is known as \textit{Hierarchical Dirichlet Process} described in next section.
%Note that it would not be useful Dirichlet process in our setting of grouped data as using DP

Another commonly used perspective of the DP is the {\it Chinese Restaurant Process} (CRP) [\cite{pitman:ln02}]
which shows that DP tends to clusters draws $\theta_{ji}$ from $G_j$. 
%CRP is known to be useful in inference and many generalizations of Dirichlet Process.
%\subsubsection*{Chinese Restaurant Process}
%\label{sec:crp}
%This process refers to draws $\theta_1, \theta_2, \hdots$ from $G$ instead of $G$. For the $i$th draw $\theta_i$, let $\phi_1, \hdots \phi_K \sim G_B$ be the values taken by previous draws. 
Let $\{\theta_{ji}\}$ denote the sequence of draws from $G_j$, and let $\{\phi_k\}$ be the atoms of $G_j$.
The CRP considers the predictive distribution of the $i$-th draw $\theta_{ji}$ given the first $i-1$ draws 
$\theta_{j1}\ldots \theta_{ji-1}$ after integrating out $G_j$: 
%Let $\theta_i$ represent the table selected by customer $i$.
%We can find the conditional distribution of $\theta_i$ conditioned on the previous observations by integrating out G. 
%$\theta_i | G \sim G , i=1 \hdots K$ 
\begin{equation}
\label{eq:crp}
\theta_{ji} | \theta_{j1}, \hdots, \theta_{ji-1}, \alpha, \bar H \sim \sum\limits_{k=1}^{K} \frac{n_{jk}}{i-1+\alpha} \delta_{\phi_k} + \frac{\alpha}{i-1+\alpha} \bar H
\end{equation}
where $n_{jk}=\sum_{i'=1}^{i-1} \delta(\theta_{ji'},\phi_k)$.
The above conditional may be understood in terms of the following restaurant analogy.
Consider an initially empty `restaurant' with index $j$ that can accommodate an infinite number of `tables'. 
The $i$-th `customer' entering the restaurant chooses a table $\theta_{ji}$ for himself, 
conditioned on the seating arrangement of all previous customers.
%A customer $i$, represented by $\theta_i$ who enters the restaurant 
He chooses the $k$-th table with probability proportional to $n_{jk}$, the number of people already seated at the table, 
and with probability proportional to $\alpha$, he chooses a new (currently unoccupied) table. 
Whenever a new table is chosen, a new `dish' $\phi_k$ is drawn ($\phi_k \sim \bar H$) and associated with the table.
The CRP readily lends itself to sampling-based inference strategies for the DP.

\vspace{0.05in}
\noindent
{\bf Hierarchical Dirichlet Process: }
%\subsection{Hierarchical Dirichlet Process}
%\label{sec:hdp}
%The Hierarchical Dirichlet Process (HDP) \cite{HDP} is a distribution over random probability measures over measure.
% ($\Theta$, $\mathcal{B}$). 
%It extends the DP to provide a model for grouped data, where multiple groups share atoms. 
%In the setting of grouped data, a set of Dirichlet process can be used with one DP corresponding to each group. To introduce sharing of atoms across groups, the  base distribution $G_B$  for the group specific Dirichlet process is itself a draw from Dirichlet process. This model is known as \textit{Hierarchical Dirichlet Process} described in next section.
Now reconsider our grouped data setting. 
If each $G_j$ is drawn independently from a DP, then w.p. 1 the atoms $\{\phi_{jk}\}_{k=1}^\infty$ for each $G_j$ are 
distinct, when $\bar H$, the base distribution is continuous.
This would mean that there is no shared topic across documents, which is undesirable.
The Hierarchical Dirichlet Process (HDP) [\cite{HDP}] addresses this problem by modeling the base distribution of the DP 
prior in turn as a draw $G_B$ from a DP, instead of the continuous distribution $\bar H$. 
Since draws from a DP are discrete, this ensures that the same atoms $\{\phi_k\}$ are shared across all the $G_j$s.
Specifically, given a distribution $\bar H$ on the space ($\Theta$, $\mathcal{B}$) and positive real numbers 
$(\alpha_{j})_{j=1}^{M}$ and $\gamma$, we denote as $\textnormal{HDP}(\bm\alpha, \gamma, \bar H)$ the following 
generative process: 
%The process defines a set of Dirichlet process $G_j$, one for each group $j$, and the base distribution $G_B$ is itself a draw from Dirichlet process DP($\gamma,H$), 
% The process defines a set of random measures $G_{j}$, one for each group $j$, and a global measure $G_B$ such that
%defined by with concentration parameters $\alpha_i, i=1 \hdots N$ and a global base measure $G_B$. $G_B$
%itself is a DP with a base measure H and concentration parameter $\gamma$ such that 
\begin{alignat}{2}
\label{eqn:hdp}
&G_B | \gamma,\bar H \sim DP(\gamma,\bar H) ~~~~~~~~~ && \notag \\
&G_j | \alpha_j,G_B \sim DP(\alpha_j,G_B) &&\forall j.
\end{alignat}
%We use the notation $\textnormal{HDP}(\bm\alpha, \gamma, H)$ to represent the above HDP process where $$G_j|\alpha_j, \gamma, H \sim \textnormal{HDP}(\alpha_j, \gamma, H) ~~~~~ \textnormal{for each }j.$$ 
When the generation of $G_j$s as described in Eqn. \ref{eqn:hdp} is followed by generation of $\{\theta_{ji}\}$ and $\{x_{ji}\}$ as in Eqn. \ref{eq:mixture}, we get the \textit{HDP mixture model}.

%$$G_B | \gamma,H \sim DP(\gamma,H)$$
%$$G_j | \alpha_j,G_B \sim DP(\alpha_j,G_B)$$
 
%\subsubsection*{Stick-breaking Construction}
Using the stick-breaking construction, the global measure $G_B$ distributed as Dirichlet process can be expressed as 
%Given the base measure $G_B$ is distributed as a Dirichlet process, it can be expressed as follows
$G_B = \sum_{k=1}^\infty \beta_k \delta_{\phi_k}$,
where the topics $\phi_k$ as before are drawn from $\bar H$ independently ($\phi_k \sim \bar H$) and the stick--breaking weights $ \bm{\beta} \sim $ GEM$(\gamma)$ represent `global' popularities of these topics.
Since $G_B$ has as its support the topics $\bm\phi$, each group-specific distribution $G_j$ necessarily has support at these topics, and can be written as follows:
\begin{equation}
\label{gJstick}
G_j = \sum_{k=1}^\infty \pi_{jk} \delta_{\phi_k};\ \ (\pi_{jk})_{k=1}^\infty \sim \textnormal{DP}(\alpha_j, \bm\beta)
\end{equation}
where $\bm\pi_j= (\pi_{jk})_{k=1}^\infty$ denotes the topic popularities for the $j$th group.

%\subsubsection*{Chinese Restaurant Franchise}
%\label{sec:crf}
%The Chinese Restaurant Franchise (CRF) Process \cite{HDP} is an analogy similar to CRP for describing the conditional distribution obtained from the HDP.
Analogously to the CRP for the DP, the Chinese Restaurant Franchise provides an interpretation of predictive distribution for the next draw from an HDP after integrating out the $G_j$s and $G_B$.
%\textbf{Computing Marginals:}
Let $\{\theta_{ji}\}$ denote the sequence of draws from each $G_j$, $\{\psi_{jt}\}$ the sequence of draws from $G_B$, and $\{\phi_k\}_{k=1}^\infty$ the sequence of draws from $\bar H$.
Then the conditional distribution of $\theta_{ji}$ given $\theta_{j1}, \hdots, \theta_{j,i-1}$ and $G_B$, after integrating 
out $G_j$ is as follows (similar to that in Eqn. \ref{eq:crp}):
\begin{equation}
\label{eqn:thetaHDP}
\begin{split}
\theta_{ji}|\theta_{j1}, \hdots, \theta_{j,i-1},  \alpha, G_B \sim
\sum_{t=1}^{m_{j\cdot}} \frac{n_{jt\cdot}}{i-1+\alpha} \delta_{\psi_{jt}} + \frac{\alpha}{n_{j\cdot\cdot}+\alpha} G_B
\end{split}
\end{equation}
where $n_{jtk}=\sum_{i'=1}^{i-1}\delta(\theta_{ji'},\psi_{jt})\delta(\psi_{jt},\phi_k)$, $m_{jk}=\sum_{t}\delta(\psi_{jt},\phi_k)$ and dots indicate marginal counts. 
As $G_B$ is also distributed according to a Dirichlet Process, we can integrate it out similarly to get the conditional distribution of $\psi_{jt}$:
\begin{equation}
\label{eqn:psiHDP}
\begin{split}
\psi_{jt}|\psi_{11},  \psi_{12}, \hdots, \psi_{21},  \hdots, &\psi_{jt-1}, \gamma, \bar H \sim\sum_{k=1}^{K} \frac{m_{\cdot k}}{m_{\cdot\cdot} + \gamma} \delta_{\phi_{k}} + \frac{\gamma}{m_{\cdot\cdot} + \gamma} \bar H
\end{split}
\end{equation}
These equations may be interpreted using a restaurant analogy with tables and dishes. 
Consider a set of restaurants, one corresponding to each group.
Customers entering each of the restaurants select a table $\theta_{ji}$ according a group specific CRP (Eqn \ref{eqn:thetaHDP}).
The restaurants share a common menu of dishes $\{\phi_k\}$.
Dishes are assigned to the tables of each restaurant according to another CRP (Eqn \ref{eqn:psiHDP}).
Let $t_{ji}$ be the (table) index of the element of $\{\psi_{jt}\}_j$ associated with $\theta_{ji}$, and let $k_{jt}$ be the (dish) index of the element of $\{\phi_k\}$ associated with $\psi_{jt}$.
%The counts of customers and tables are maintained. 
%Then $n_{jtk}$ denotes the number of customer in restaurant $j$ sitting at table $t$ having disk $k$, and $m_{jk}$ denotes the number of tables in restaurant $j$ serving disk $k$. 
%Note that marginal counts are represented with dots.
Then the two conditional distributions above can also be written in terms of the indexes $\{t_{ji}\}$ and $\{k_{jt}\}$ instead of referring to the distributions directly.
If we draw $\psi_{jt}$ via choosing a summation term, we set $\psi_{jt}= \phi_k$ and let $k_{jt} =k$ for the chosen $k$. 
If the second term is chosen, we increment $K$ by 1 and draw $\phi_K \sim \bar H$ and set $\psi_{jt} = \phi_K$ and $k_{jt} = K$.
This CRF analogy leads to efficient Gibbs sampling-based inference strategies for the HDP mixture model [\cite{HDP}].

%If we draw $\psi_{jt}$ via choosing a summation term, we set $\psi_{jt}= \phi_k$ and let $k_{jt} =k$ for the chosen $k$. If the second term is chosen, we increment $K$ by 1 and draw $\phi_K \sim H$ and set $\psi_{jt} = \phi_K$ and $k_{jt} = K$.

%To use the equations to obtain samples of $\theta_{ji}$, first for each $j$ and $i$, sample $\theta_{ji}$ using (\ref{thetaHDP}). If a new sample of $G_B$ is needed, we use (\ref{psiHDP}) to obtain a new sample $\psi_{jt}$ and set $\theta_{ji} = \psi_{jt}$.

%Consider a set of restaurants that share a menu dishes. at each restaurant, a customer sits at a table and all the customers at a table share the dish ordered by the first customer sitting on the table.

%Each restaurant corresponds to a group and the customers correspond to the factors $\theta_{ji}$. Let $\phi_1, \hdots, \phi_K \stackrel{iid}{\sim} H$ denote the global menu of dishes. Let the notation $\psi_{jt}$ denote the dish served at the table $t$ of restaurant $j$.
%\subsection{Nested Dirichlet Process}

\vspace{0.05in}
\noindent
{\bf Nested Dirichlet Process: }
%In other applications of grouped data, we may want to cluster observations in each group by fitting a mixture model, by
%learning a group specific distribution over the various mixture components in each group, and simultaneously cluster these group 
%specific distributions inducing a clustering over the groups themselves. 
In other applications of grouped data, we may want to cluster observations in each group by learning group specific 
a mixture distributions and simultaneously cluster these group specific distributions inducing a clustering over 
the groups themselves. 
%additionally be interested in clustering the data groups themselves
For example, when analyzing patient records in multiple hospitals, we may want to cluster the patients in each hospital 
by learning a distribution over patient profiles and cluster hospitals having the same distribution over patient profiles. 
%as well in terms of the profiles of patients coming there.
The HDP cannot do this, since each group specific mixture $G_j$ is distinct.This problem is addressed by the nested Dirichlet Process [\cite{NDP}].

This problem is addressed by the nested Dirichlet Process [\cite{NDP}], which first defines a set $\{G^0_{r}\}_{r=1}^\infty$ of 
 distributions  with an infinite support:
\begin{equation}
\label{nDPsb1}
G^0_{r} = \sum_{k=1}^\infty \pi^0_{rk}\delta_{\phi^0_{rk}},\ \phi^0_{r k}\sim \bar H,\ \{\pi^0_{r k}\}_{k=1}^\infty \sim GEM(\gamma^0)
\end{equation}
and then draws the group specific distributions, that we now term as $G^1_j$, from a mixture over 
these set of $\{G^0_{r}\}$:
$$G^1_j\sim G^1_B\equiv\sum_{r=1}^\infty \beta^0_{r} \delta_{G^0_{r}},\ \{\beta^0_{r}\}\sim GEM(\gamma^1)$$
We denote the generation process as $\{G^1_j\}\sim nDP(\gamma^1,\gamma^0,\bar H).$ 
The process ensures non-zero probability of different groups selecting the same $G^0_{r}$, 
leading to clustering of the groups
themselves.
Using Eqn. \ref{eq:dpStick}, the draws $\{G^1_j\}$ can be characterized as: 
\begin{equation}
\label{eqn:ndp:dpfamily}
G^1_j\sim G^1_B,\ G^1_B \sim DP(\gamma^1, DP(\gamma^0, \bar H)) 
\end{equation}
where the base distribution of the outer DP is in turn another DP, unlike the HDP where it is DP distributed.
Thus the nDP can be viewed as a distribution on the space of distributions on distributions. 

The nDP can be expressed with the following restaurant analogy with two levels of restaurants. 
Each group (hospital/document) is associated with an `outer' level $1$ restaurant while each distribution 
$G^0_{r}$ corresponds to an `inner' level $0$ restaurant.
Each outer restaurant picks a distribution $G^1_j$, through picking a 'dish' from a 
global menu of dishes across outer restaurants based on the dish's popularity according to $G^1_B$.
Each dish in this menu, that corresponds to a unique inner restaurant, defines a specific distribution over patient profiles.
Hence each outer restaurant gets a distribution corresponding to one of the inner restaurants through this process, 
leading to a grouping of the outer restaurants (hospitals) based on the inner restaurant (distribution over patient profiles) chosen.
The $i^{th}$ customer entering an outer restaurant $j$ goes to the corresponding inner restaurant, with index $r$,  
such that $G^1_j=G^0_{r}$. Now the customer selects a table in this restaurant, with the index, say, $k$. 
The data is generated from the corresponding $F(\phi_{r k})$.

\textit{\underline {A Note on Notation:}} nDP brings to focus the idea of nesting, where the the distributions 
at one level ($\{G^0_r\}$ at level 0) are themselves atoms for the next level (level 1 mixture distribution $G^1_B$).
Hence, with the nDP, we introduce  the notion of \textit{levels} into 
our notation through superscripts for random variables. For the rest of the paper \textit{the superscript of a random variable 
 indicates the level} of the variable. %Further, we use $i$ to indicate the index of the observed data within a specific group 
% $j$. We use $r$ as a restaurant index and $k$ as a dish index.
Table \ref{tab:notation} shows a ready summary of the notation used through the 
rest of the paper.

{\bf Nested Chinese Restaurant Process: }
The nDP can be viewed as a tool for building a non-parametric mixture of mixtures.
The Nested Chinese Restaurant Process (nCRP) [\cite{blei:jacm10}],  
is a closely related model for multi-level clustering. 
The nCRP extends CRP by creating an infinitely-branched tree structure over restaurants to define a distribution over 
finite length paths of trees for modeling topic hierarchies from documents. 
The nCRP can be interpreted with a restaurant analogy consisting of multiple levels of restaurants as follows as described in 
[\cite{blei:jacm10}].
`` \textit{A tourist arrives at the city for an culinary vacation. On the first evening, he enters
the root Chinese restaurant and selects a dish using the CRP distribution, based on its popularity (equation \ref{eq:crp}).
On the second evening, he goes to the restaurant identiﬁed on the first nightís dish
and chooses a second dish using a CRP distribution based on the popularity of the dishes
in the second night\'s restaurant. He repeats this process forever.}''
The nCRP however is closely connected to the nDP since a K-level nCRP can be obtained by integrating out the DP at each level in 
a K-level nDP facilitating multi-level non-parametric mixture models.

{\bf Multi-level Admixture models:} 
The nDP enables modeling a non-parametric mixture of non-parametric mixtures, while the nCRP provides a 
hierarchical prior for multilevel non-parametric mixture models. In other words, the multi-level nDP leads to a prior where 
each distribution at a specific level $l$, is a mixture over a distinct set of distributions at the previous level $l-1$. 
Hence, there are no atoms in common between distributions at each level. The nDP and multi-level nDP 
are therefore not suited for applications that require mixture components to be shared across group specific distributions at 
each level. Several real world scenarios are however more effectively modeled by 
\textit{\bf \textit{multi-level admixture models}}
where each level has a group of distributions which share mixture components.

A example of \textit{entity-topic modeling} for document collections clearly illustrates the 
limitation of existing models. Here, we would like to model documents as having distributions over a set of 
latent entities, with multiple documents sharing entities. We would like to model the entities themselves as 
distributions over a set of latent topics, with the ability for multiple entities to share topics. 
This constitutes a two level admixture model, where group specific distributions at one level 
(the 'entity' distributions over topics) must share atoms (topics), which are themselves distributions at the previous level 
(the 'topic' distribution over words).

The author-topic model (ATM) [\cite{AT}], an extension of LDA, captures this modeling scenario for the parametric case
where the entities(authors) for each document are observed and the number of topics is known in advance.
Consider a corpus containing  $A$ authors. The ATM captures {\it known} authors, $A_j\subset A$ of each document, 
by modeling documents as a \textit{uniform} distributions $\{G^1_j\}$ over corresponding sets of authors $A_j$ and 
authors as distributions $\{G^0_r\}$ over $K$ topics. The words are generated by first sampling one of the 
known authors $\theta^1_{ji}$ of the document (with  $z^1_{ji}$ holding the global index of this author), 
followed by sampling a topic $\theta^0_{ji}$ from the topic distribution of that author :
%(with $z^0_{ji}$ holding the global index of this topic):
\begin{eqnarray}
\theta^1_{ji} |G^1_j \sim G^1_j; ~ \theta^0_{ji} |G^0_j, z^1_{ji}=r  \sim G^0_r; ~
\ x_{ji} |\theta^0_{ji} \sim F(\theta_{ji}),\ \forall j,i
\label{eqn:atm}
\end{eqnarray}

The ATM however cannot handle a more realistic scenario of non-parametric modeling where the number of topics is not 
fixed in advance and author set for each document is not fully observed. Such an application calls for 
{\bf multi-level non-parametric admixture modeling}, a previously unexplored problem.
Motivated by this, we propose the nested Hierarchical Dirichlet Process(nHDP) for multi-level non-parametric admixture modeling.
%to address the problem of 
%multi-level non-parametric admixture modeling

% 
% In addition, if we want to model the topics as non-parametric distribution over words, where the number of words is 
% not known apriori, with different topics sharing words, this could lead to a three level non-parametric admixture model, 
% motivating the need for multi-level non-parametric admixture modeling. 
% We present the Nested Hierarchical Dirichlet Processes (nHDP) for multi-level non-parametric admixture modeling. 
% 

\section{Nested Hierarchical Dirichlet Processes}
In this section, we introduce the Nested Hierarchical Dirichlet Processes. 
%To better motivate the need for multi-level admixture modeling using nested HDP,
For this, we first introduce 2-nHDP i.e. the two level nested HDP for non-parametric modeling of entities and topics 
and then generalize this to L-nHDP for any given number of L levels.

\subsection{Two-level Nonparametric Admixture Model}
\label{sec:model}
 
%Non-parametric entity topic analysis involves simultaneously discovering an unknown or a partially 
%observed set of entities and an unknown number of topics from document collections. 
Recall that in [\cite{AT}], the authors approach the problem of modeling the topics and entities for the application of 
author-topic modeling by taking a two level approach. 
%They consider a document as a distribution over entities i.e. authors for the given application and entities as distributions over topics. They make an assumption that all entities in a document are observed and the number of topics is known apriori.
%We consider a non-parametric setting where the entities and the number of topics are not known apriori. 
%Further, we would like document specific distributions over entities to have entities in common and entity specific 
%distributions over topics to have topics in common. This  necessitates the need for two-level non-parametric admixture 
%modeling. We assume that the topics themselves are drawn from a Dirichlet prior, as distributions over a finite set of words in the vocabulary. 
Our aim is to build a 2-level admixture 2-nHDP for a non-parametric treatment of this problem. 
%we now systematically build a 2-nHDP, a two level non-parametric 
%admixture model for this entity-topic modeling usecase. 
However, before this, we first present a simpler intermediate model which we call DP-HDP, an extension of nDP, 
for ungrouped data, where the words are not grouped into documents, leading to a mixture of admixture model.
(This can also be interpreted as a usecase for single document analysis instead of a collection of documents). 
We then gradually extend it for grouped data (multiple documents) to build 2-nHDP modeling non-parametric admixtures of admixtures. 
%where each group or document 
%is associated with a single entity, leading to a mixture of admixture model and then gradually extend it for multiple 
%entities in each document to build the two-level nHDP modeling non-parametric admixtures of admixtures. 
We next generalize this to (L+1)-nHDP in section \ref{sec:multilevel}.

%\subsection{Single-Entity Documents}
%\label{single}
\vspace{0.05in}
\noindent
{\bf DP-HDP for Ungrouped Data: }
Consider an entity-topic modeling scenario where the observed data i.e. set of words is not grouped as documents. One could conceive performing such two-level modeling for such data with the nDP. In nDP, entities are $\{G^0_r\}_{r=1}^\infty$ of equation \ref{nDPsb1} with $\phi$ as the topic variables drawn from a base distribution $\bar H$.
However, the nDP is unsuitable for such analysis, since the entities drawn from a DP, with 
a continuous base distribution $\bar H$ , do not share topic atoms.
This can be modified by first creating a set of entities  $\{G^0_{r}\}_{r=1}^\infty$ such that 
they share topics. One way to do this %while retaining flexibility
is to follow the HDP construction for entities: 
\begin{equation}
\label{gS}
G^0_{r}\sim HDP(\{\alpha^0_{r}\},\gamma^0,\bar H), r=1\ldots\infty
\end{equation}
This can be followed by drawing the entity for each word $i$ from a mixture over the $G^0_{r}$s:
\begin{equation}
\label{gJd}
G^1_i\sim G^1_B \equiv \sum_{r=1}^\infty \beta^1_{r}\delta_{G^0_{r}},\ \beta^1\sim GEM(\gamma^1)
\end{equation}
This may be interpreted as creating a countable set of entities $\{G^0_{r}\}$ by defining topic preferences 
(distributions over topics) for each of them, and then defining a `global popularity' $G^1_B$ of the entities.
Using Eqn. \ref{eq:dpStick}, we observe that $G^1_B \sim DP(\gamma^1,HDP($ $\{\alpha^0_{r}\},\gamma^0,\bar H))$.
Observe the relationship with the nDP (Eqn. \ref{eqn:ndp:dpfamily}).
Like nDP, this also defines a distribution over the space of distributions on distributions.
But, instead of a DP base distribution for the outer DP, we have achieved sharing of topics using a HDP base distribution.
We will write $G^1_i\sim \mbox{DP-HDP}(\gamma^1,\{\alpha^0_{r}\},\gamma^0,\bar H)$.

Note that multiple words can choose the same entity. As before, entity $G^1_i$ can now be used as prior for sampling topics, say $\{\theta^0_{i}\}$ for 
individual words which chose that entity, using
\begin{equation}
\label{eq:DP-HDPmm}
\theta_{i}\sim G^1_i,\ x_{j}\sim F(\theta_{i})
\end{equation}
We will call this the DP-HDP mixture model. 
Note that one can also alternatively use this model for grouped data where each group or document is associated with a single entity and each word in the document chooses topic as per the entity distribution over topics.
\vspace{0.05in}
\noindent
{\bf 2-nHDP for Grouped Data: }	
In this section, we extend the earlier model for grouped data since most of the applications use multiple documents e.g. in the form of news articles, scientific literature, images, etc. 

We extend the approach presented in \S~\ref{single} to the setting of grouped data since most applications use multiple documents e.g. news articles, scientific literature, images, etc.
 In the single entity model, since a document is associated with one entity, a single entity is sampled for all the words in the document. Now, in the case of multiple entities per document, first we sample an entity for each word in the document, and then a topic is sampled according to the entity specific distribution of topics. 

As in the previous model, we first create a set of entities $\{G^0_{r}\}_{r=1}^\infty$ as distributions over a common set of topics $\{\phi_{k}\}_{k=1}^\infty$ ($\phi_k\sim \bar H$) by drawing independently from an HDP (Equation \ref{gS}), 
%\begin{equation}
%G_{k'}\sim HDP(\{\alpha_{k'}\},\gamma,H)
%\end{equation}
and then create a global mixture $ G^1_B$ over these entities (Equation \ref{gJd}).

Earlier in the absence of groupings, this global popularity was used to sample entities for all the words.
Now, for each document $j$, we define a local popularity of entities, derived from their global popularity $G^1_B$:
%This is defined as 
\begin{equation}
\label{eqn:ndpLocalPopularity}
G^1_j \equiv \sum_{r=1}^{\infty} \pi^1_{jr}\delta_{G^0_{r}},\ \{\pi^1_{jr}\}\sim DP(\alpha^1_j,\beta^1)
\end{equation}
Now, sampling each factor $\theta^0_{ji}$ in group $j$ is preceded by choosing an entity $\theta^1_{ji}\sim G^1_j$ 
by sampling according to local entity popularity $G^1_j$. Note that $P(\theta^1_{ji}=G^0_{r})=\pi^1_{jr}$. 

Note that the above equation \ref{eqn:ndpLocalPopularity} is similar to the stick breaking definition of HDP in Equation \ref{gJstick}. We can see that $G^1_j$ is drawn from a HDP with the base distribution over atoms $\{G^0_{r}\}$ instead of topics $\{\phi_k\}$. This distribution over  $\{G^0_{r}\}$ is again an HDP.
Therefore, we can write:
\begin{equation}
\label{eqn:ndp}
\theta^1_{ji} \sim G^1_j \sim \mbox{HDP}(\{\alpha^1_j\},\gamma^1,\mbox{HDP}(\{\alpha^0_{r}\},\gamma^0,\bar H))
\end{equation}
We refer to the two HDPs as the inner and outer HDPs and hence, call this as 2-nHDP. We can write $\theta^1_{ji} \sim  2-HDP(\{\alpha^1_j\}, \gamma^1, \{\alpha^0_{r}\}, \gamma^0, \bar H)$.
Similar to the nDP and the DP-HDP (Eqn. \ref{gJd}), this again defines a distribution over the space of distributions 
over distributions. The 2-HDP mixture model is completely defined by subsequently sampling $\theta^0_{ji} \sim \theta^1_{ji}$, 
followed by $x_{ji} \sim F(\theta^0_{ji})$.

An alternative characterization of the 2-nHDP mixture model is using the topic index $z^0_{ji}$ and entity index $z^1_{ji}$ 
corresponding to $x_{ji}$:

\begin{eqnarray}
\label{eq:mixtureMulti}
\beta^0 \sim GEM(\gamma^0);\ \pi^0_{r} \sim DP(\alpha^0, \beta^0) ;\ \phi_k \sim \bar H,\ k,r=1\ldots\infty \nonumber  \\
\beta^1 \sim GEM(\gamma^1)\ ; \pi^1_{j} \sim DP(\alpha^1_j, \beta^1),\ j=1\ldots M \nonumber \\
z^1_{ji} \sim \pi^1_j\ ;\ z^0_{ji} \sim \pi^0_{z^1_{ji}} ;\ x_{ji} \sim F(\phi_{z^0_{ji}}),\ i=1\ldots n_j
\end{eqnarray}

This may be understood as first creating a entity-specific distributions $\pi^0_{r}$ over topics using global 
topic popularities $\beta^0$, followed by creation of document-specific distributions $\pi^1_j$ over entities using 
global entity popularities $\beta^1$.
Using these parameters, the content of the $j^{th}$ document is generated by sampling repeatedly in $iid$ fashion an 
entity index $z^1_{ji}$ using $\pi^1_j$, a topic index $z^0_{ji}$ using $\pi^0_{z^1_{ji}}$ and finally a word 
using $F(\phi_{z^0_{ji}})$.

Observe the connection with the ATM in Eqn. \ref{eqn:atm}.
The main difference is the the set of entities and topics is infinite. 
Separately, each document now has a distinct non-uniform distribution $\pi^1_j$ over entities.

(Move the following to/before background....?)

Also, observe that we have preserved the HDP notation to the extent possible, to facilitate understanding.
To distinguish between variables corresponding to the two HDPs levels in this model, we use the 
superscript $0$ for symbols corresponding to the the inner HDP modeling entities as distributions as topics and 
superscript $1$ for symbols corresponding to the the outer HDP modeling documents as distributions over entities. 
Going forward, we follow the same convention for naming variables in the multi-level HDP with multiple levels 
of nesting.

\subsection{Multi-level Non-parametric Admixture modeling}
\label{sec:multilevel}
We now present (L+1)-nHDP, a generalized extension to 2-HDP proposed in the previous section \ref{sec:model}, that can be used for multi-level non-parametric admixture modeling. 

The 2-nHDP was constructed by first creating a set of entities, $\{G^0_{r}\}_{r=1}^{\infty}$ 
by drawing each of these distributions from an inner HDP with base distribution $\bar H$. 
This is followed by drawing document specific distributions at the outermost level
$\{G^1_{j}\}_{j=1}^{M}$ from the outer HDP, with the base distribution as the inner HDP. 
To extend this to multiple levels, at each level, we draw group level distributions from an HDP with the base distribution 
at as the previous level HDP. 

Let $L+1$ denote the number of levels of nesting, indexed by $l\in\{0,\hdots,L\}$. Through the rest of this section, 
the superscript of a random variable denotes the level of the random variable. 
The nested HDP comprises of multiple levels of HDPs, where the base distribution of HDP at level $l$ is the the HDP at level $l-1$.
The innermost level is 0 while the 
outer most level is $L$. The groups in the outermost level $L$ correspond to documents in the case of entity topic modeling. 
%The outermost level (L) HDP models the document specific distribution over the outermost level atoms. 
At the inner most level 0, we have a HDP, with a base distribution $\bar H$ from which the inner most level entities are drawn. 
In the case of entity topic modeling these inner level entities are topics that are modeled as a distributions over words. 

At level 0, the inner most level, we draw level-1 entities $\{G^0_{r}\}_{{r}=1}^{\infty}$ from a HDP with base distribution $\bar H$. 
This step corresponds to equation \ref{gS} of the 2-nHDP and 
constitutes a non-parametric admixture over atoms drawn from $\bar H$. Note that in case of two-level models, we had termed $\{G^0_{r}\}$ as entities. In case of this multi-level model, we term these entities as level-1 entities and topics can be considered as level-0 entities. 
Hence, at level 0, we have
% \begin{equation}
% 
% 
% \end{equation}
% This can be followed by drawing each group specific distribution from a mixture over the $G_{k'}$s:
% \begin{equation}
% \label{gJd}
% G'_j\sim G'_0 \equiv \sum_{k'=1}^\infty \beta'_{k'}\delta_{G_{k'}},\ \beta'\sim GEM(\gamma')
% \end{equation}
% 
\begin{gather}
G^0_{r}\sim HDP(\{\alpha^0_{r}\},\gamma^0,\bar H) \equiv G^0_B, \text{ $r=1,\hdots$ }  \\ \nonumber
\text{Alternately expressed as, } \beta^0 \sim GEM(\gamma^0); G^0_B= \sum_{k=1}^{\infty} \beta^0_k \delta_{\phi_k} \\ \nonumber
G^0_{r} = \sum_{k=1}^{\infty} \pi^0_{r k} \delta_{\phi_k} \text{ where } \{ \pi^0_{r k} \} \sim DP(\alpha^0_{r}, \beta^0), \text{ $r=1,\hdots$ }
\end{gather}
We denote the HDP distribution itself at level $0$ by $H^0$, which subsequently
becomes the base distribution for next level HDPs.
At any level $l\in\{1,2,\hdots,L-1\}$, $H^{l-1}$ becomes the base distribution of the $l^{th}$ level HDP, 
while the group level distributions at the previous level, $G^{l-1}_k, k \in {1,2,\hdots}$, become the atoms for 
the group level distributions that we construct at the $l^{th}$ level, $G^l_{r} , r=1,2, \hdots,$
\begin{gather}
G^l_{r}\sim HDP(\{\alpha^l_{r}\},\gamma^l,H^{l-1}) \equiv H^l, \text{ $r=1,\hdots$ }\\ \nonumber
\text{Alternately expressed as, } \beta^l \sim GEM(\gamma^l) \text{, and } G^l_B= \sum_{k=1}^{\infty} \beta^l_k \delta_{G^{l-1}_k} \\ \nonumber
G^l_{r} = \sum_{k=1}^{\infty} \pi^l_{r k} \delta_{G^{l-1}_k} \text{ where } \{ \pi^l_{r k} \} \sim DP(\alpha^l_{r}, \beta^l), \text{ $r=1,\hdots$ }
\end{gather}
For the HDP at the outermost level $L$, the base distribution is  $H^{L-1}$, the HDP from the previous level. 
At this level we have a set of M groups, that correspond to the number of documents in the case of document modeling. 
While it is possible to develop a multilevel admixture model where the number of groups is unobserved at every level, 
in this paper, we assume the number of groups at the outermost level to be an observed quantity in a fashion aligned 
with the document modeling usecase. Hence, at level $L$, we have, 
\begin{gather}
G^L_{j}\sim HDP(\{\alpha^L_{j}\},\gamma^L,H^{L-1}) \equiv H^{L}, \text{ $r=1,\hdots$ } \\ \nonumber
\text{Alternately expressed as, } \beta^L \sim GEM(\gamma^L) \text{, and } G^L_B= \sum_{k=1}^{\infty} \beta^L_k \delta_{G^{L-1}_k} \\ \nonumber
G^L_{j} = \sum_{j=1}^{\infty} \pi^L_{j k} \delta_{G^{L-1}_k} \text{ where } 
\{ \pi^L_{j k} \} \sim DP(\alpha^L_{j}, \beta^L) , \text{ $r=1,\hdots$ }
\end{gather}
Each observed data item $i\in{1,\hdots, N_j}$ that resides with one of the outermost groups $j\in\{1,\hdots,M\}$ is now 
associated with an entity (group level distribution) from each HDP level $l$ , which itself is a distribution over entities 
drawn from the previous level $l-1$ HDP. 
Hence we generate the data as follows. First generate $\theta^L_{ji} \sim G^L_j$ from the group level distribution 
at the outermost group $j$. For any level $l\in\{L-1,\hdots,0\}$, we select $\theta^l_{ji} \sim \theta^{l+1}_{ji}$. 
Note that  $\theta^l_{ji} \sim \theta^{l+1}_{ji}$ thus sampled is equal to one of the $\{ G^{l-1}_k\}_{k=1}^{\infty}$
variables, (which are themselves distributions over atoms drawn from previous level HDP). 
%For levels $\{l\in\{L,\hdots,2\}$, and 
$\theta^0_{ji}$ is equal to one of 
$\{\phi_k\}_{k=1}^{\infty}$ at the inner most level zero. 
Finally data items are generated as $x_{ji} \sim F(\theta^0_{ji})$. 

Similar to the 2-nHDP, (L+1)-nHDP can be defined using the index $z^l_{ji}$ of the atom $\theta^l_{ji}$ 
 at each level $l$ corresponding to data item $x_{ji}$ as follows.
 \begin{eqnarray}
\label{eq:mixtureMulti}
\beta^0 \sim GEM(\gamma^0);\ \pi^0_{r} \sim DP(\alpha^0, \beta^0) ;\ \phi_k \sim \bar H,\ k,r=1\ldots\infty \nonumber  \\
\beta^l \sim GEM(\gamma^l)\ ; \pi^l_{r} \sim DP(\alpha^l_r, \beta^l),\ r=1\ldots\infty, l=1\ldots L-1 \nonumber \\
\beta^L \sim GEM(\gamma^L)\ ; \pi^L_{j} \sim DP(\alpha^L_j, \beta^L),\ j=1\ldots M \nonumber \\
z^L_{ji} \sim \pi^L_j\ ;\ z^l_{ji} \sim \pi^l_{z^l_{ji}} ;\ x_{ji} \sim F(\phi_{z^0_{ji}}),\ i=1\ldots n_j, l=1\ldots L-1
\end{eqnarray}

% With this prior, we now describe the generative process for the observed data. 
% In each group $j\in\{1,\hdots,M\}$ at the outermost level M, the observed data  $Xji$, $i \in \{1, \hdots, N_j\}$
% is generated as follows by generating entities of different levels for each 
% observation $i$ in each group $j$
% 
% \begin{gather*}
%   \theta^L_{ij} \sim G^L_j \\ \nonumber
%   \theta^{L-1}_{ji} \sim \theta^{L}_{ji} \\ \nonumber
%   . \\ \nonumber
%   . \\ \nonumber
%   \theta^l_{ji} \sim \theta^{l+1}_{ji} \\ \nonumber
%   . \\ \nonumber
%   . \\ \nonumber
%   \theta^0_{ji} \sim \theta^1_{ji}  \\ \nonumber
%   X_ji | \theta^0_{ji} \sim F(\theta^0_{ji})
% \end{gather*}

\subsection{Nested Chinese Restaurant Franchise}
\label{nCRF}
\vspace{0.05in}
\noindent
%{\bf Nested Chinese Restaurant Franchise: }
In this section, we derive the predictive distribution for the next draw $\theta^l_{ji}$ at various levels from the nHDP 
given previous draws, after integrating out the various group level distributions $\{G^l_r\}$ and $G^l_B$ at each level. 
We also provide a restaurant analogy  for the nHDP in terms of multiple levels of nested CRFs, 
corresponding to the multiple levels of HDP.
This will be useful for the inference algorithm that we describe in Section \ref{sec:multi-inference}.

We start with the outermost level L. 
Let $\{\theta^L_{ji}\}_{i=1}^{N_{j}}$ denote the sequence of draws from $G^L_j$, and $\{\psi^L_{jt}\}_{t=1}^{m^L_{j\cdot}}$ 
denote the sequence of draws from $G^L_B$.
Then the conditional distribution of $\theta^L_{ji}$ given all previous draws after integrating out $G^L_j$ looks as follows:
\begin{equation}
\label{thetaLmulti}
\begin{split}
\theta^L_{ji}|\theta^L_{j1:i-1}, \alpha^L_{j}, G^L_B \sim
&\sum_{t=1}^{m^L_{j\cdot}} \frac{n^L_{jt}}{i - 1+\alpha^L_{j}} \delta_{\psi^L_{jt}} + \frac{\alpha^L_{j}}{i - 1 +\alpha^L_{j}} G^L_B
\end{split}
\end{equation}
where $n^L_{jt} = \sum_i \delta(\theta^L_{ji}, \psi^L_{jt})$, $m^L_{jk}=\sum_{t}\delta(\psi^L_{jt}, G^{L-1}_{k})$. 
%A draw from the above mixture can be obtained by drawing from terms on RHS with probabilities given by the corresponding mixing proportions. 
Next, we integrate out $G^L_B$, which is also distributed according to Dirichlet process:
\begin{equation}
\label{psiLmulti}
\begin{split}
\psi^L_{jt}|\psi^L_{11}, & \psi^L_{12}, \hdots, \psi^L_{21},  \hdots, \psi^L_{j,t-1}, \gamma^L, H^{L-1} \sim\\ &
\sum_{k=1}^{K^L} \frac{m^L_{\cdot k}}{m^L_{\cdot\cdot} + \gamma^L} \delta_{G^{L-1}_{k}} + \frac{\gamma^L}{m^L_{\cdot\cdot} + \gamma^L} H^{L-1}
\end{split}
\end{equation}

%To use the above equations to obtain samples of $\theta'_{ji}$, first sample $\theta'_{ji}$ for each $j$ and $i$ using (Eqn \ref{thetaMulti}). If new sample from $G'_0$ is needed, we use (Eqn \ref{psiMulti}) to obtain a new sample $G_{k'}$. Therefore, each $\theta'_{ji}$ variable gets a $G_{k'}$ assigned to it.
Note that $K^L$ here refers to the number of unique atoms $G^{L-1}_{k}$ already drawn from the base HDP of $H^{L-1}$.
Observe that each $\theta^L_{ji}$ variable gets assigned to one of the $G^{L-1}_{k}$ variables, from which $\theta^{L-1}_{ji}$
is drawn (recall $\theta^{L-1}_{ji} \sim \theta^L_{ji}$). Hence, the predictive 
distribution for $\theta^{L-1}_{ji}$, given  $\theta^L_{ji}=G^{L-1}_{k}$  is obtained by integrating out the 
corresponding grouplevel distribution $G^{L-1}_{k}$. Similarly, for any general level $l$, 
  given that $\theta^{l+1}_{ji}=G^l_k$, $\theta^{l}_{ji} \sim \theta^{l+1}_{ji}$, is drawn by integrating out 
 the group level distribution $G^{l}_{k}$. Hence,  for level $l\in\{L-1,\hdots, 1\}$, let 
 $\theta^l_{r:ji} \equiv \{\theta^l_{j'i'}: \theta^{l+1}_{j'i'} = G^l_r,\ \forall i',\ j' \leq j, \text{ and } i' < i,\ j' = j\}$
 denote the sequence of previous draws from $G^l_r$. Hence, 

 \begin{equation}
\label{thetalmulti}
\begin{split}
\theta^l_{ji}|\theta^{l+1}_{ji}=G^l_r, \theta^l_{r:ji},  \alpha^l, G^l_B \sim
\sum_{t=1}^{m^l_{r\cdot}} \frac{n^l_{rt\cdot}}{i-1+\alpha^l} \delta_{\psi^l_{rt}} + \frac{\alpha^l}{n^l_{r\cdot\cdot}+\alpha^l} G^l_B
\end{split}
\end{equation} 
where $n^l_{rt}=\sum_{\theta^l_{j'i'} \in \theta^l_{r:ji}} \delta(\theta^l_{j'i'},\psi^l_{rt})$, 
the number of times component $\psi^l_{rt}$ was picked. 
As $G^l_B$ is also distributed according to a Dirichlet Process, we can integrate it out similarly and write 
the conditional distribution of $\psi^l_{rt}$ as follows with $m^l_{ r k}=\sum_t \delta(\psi^l_{r t}, G^{l-1}_{k})$,
and $H^{l-1}$ is the previous level HDP : 
\begin{equation}
\label{psilmulti}
\begin{split}
\psi^l_{rt}|\psi_{11},  \psi_{12}, \hdots, \psi_{21},  \hdots, &\psi_{rt-1}, \gamma, H^{l-1} \sim 
\sum_{k=1}^{K^l} \frac{m^l_{\cdot k}}{m_{\cdot\cdot} + \gamma^l} \delta_{G^{l-1}_{k}} + \frac{\gamma^l}{m^l_{\cdot\cdot} + \gamma^l} H^{l-1}
\end{split}
\end{equation}
At level 0, the predictive distribution for $\theta^0_{ji}$, given $\theta^1_{ji}=G^0_r$ can be obtained by integrating 
out $G^0_r$ replacing $l$ with 0 in equation \ref{thetalmulti}. Similarly, the predictive distribution for $\{\psi^0_{rt}\}$, 
draws from $G^0_B$, can be obtained by integrating out $G^0_B$ as follows. 
\begin{equation}
\label{psi0multi}
\begin{split}
\psi^0_{rt}|\psi^0_{11}, & \psi^0_{12}, \hdots, \psi^0_{21},  \hdots, \psi^0_{r,t-1}, \gamma^0, \bar H \sim
\sum_{k=1}^{K^0} \frac{m^0_{\cdot j}}{m^0_{\cdot\cdot} + \gamma^0} \delta_{\phi_{k}} + \frac{\gamma^0}{m^0_{\cdot\cdot} + \gamma^0} \bar H
\end{split}
\end{equation} 
At this level, each $\theta^0_{ji}$ is assigned to a $\phi_{k}$ that are drawn from $\bar H$, the base distribution of the nHDP. 
Given the $\phi_{k}$ that corresponds to $\theta^0_{ji}$, the observed data is generated as $F(\phi_{k})$.
Note that each of the conditional distributions for $\theta^l_{ji}$ and $\psi^l_{rt}$ are similar to that for 
CRF (Eqns. \ref{eqn:thetaHDP} and \ref{eqn:psiHDP}). We interpret these distributions as a 
\textit{nested Chinese Restaurant Franchise} (nCRF), involving CRFs with multiple levels of nesting. 

We now describe in detail the restaurant analogy for the nested Chinese Restaurant Franchise. The nCRF comprises of  
multiple levels of CRF. At each level $l$, there exist multiple restaurants $\{G^l_{r}\}$, each containing 
a countably infinite number of tables.
Each table $t$ in restaurant $r$ of level $l$ is associated with a dish $\psi^l_{rt}$ from global menu of 
dishes specific to that level. 
$\{G^l_{B}\}$ is the distribution over the dishes in the global menu at level $l$ modeling the global popularity of the dishes. 
 %from which all the dish assignments $\psi^l_{rt}$ are drawn for each table of restaurant $r$ 
 %according to equation \ref{psilmulti}. 
 
Imagine a customer on a culinary vacation. We trace the journey of this customer to show the process of generating $x_{ji}$, 
the $i^{th}$ word in the $j^{th}$ document through the dishes he selects at restaurants at various levels. 
The customer first enters the restaurant $j$ in the outermost level $L$ as the
$i^{th}$ customer and choses a table with index $t^L_{ji}$, based on the popularity of the table governed by $G^L_{j}$. 
Each table $t$ in this level $L$ restaurant $j$ is associated with a dish $\psi^L_{jt}$ from a global menu at level L.
%Let $\{{\phi^L_k}\}_{k=1}^{\infty}$ be the set of dishes in the global menu at level $L$. 
Each of these dishes 
has a one-to-one correspondence with a unique restaurant at level $L-1$, leading to \textit{nesting} between CRF levels. 
We use the variable $\theta^L_{ji}=\psi^L_{jt_{ji}}$ to denote the level $L$ dish thus chosen by the customer, through his table selection, 
and $z^L_{ji}$ to denote the index of the dish within the global menu and $r^L_{ji}$ to denote the level $L-1$ 
restaurant corresponding to the dish chosen. The customer now enters the restaurant $r^L_{ji}$ at level $L-1$ and repeats this process by selecting a table based on 
the distribution $G^{L-1}_{r^L_{ji}}$.

At any intermediate level $l$, the customer enters the restaurant $r^{l+1}_{ji}$, governed by the dish chosen at the previous 
level. He then selects a table $t^l_{ji}$. Each table $t$ in this restaurant has a dish $\psi^l_{rt}$ from the global 
level $l$ menu governing the dish $\theta^l_{ji}$ chosen by the customer.Each dish $k$ in the global menu corresponds to a 
unique restaurant in the previous level. This process continues where at level 0, the customer enters restaurant $r^{1}_{ji}$
governed by the dish selected in level $1$. The customer then chooses a table $t^0_{ji}$ which is associated with a dish
$\psi^0_{rt}=\phi^0_k$, say for some $k \in \{1,2,\hdots\}$. The word $x_{ji} \sim F(\phi_k)$ is generated from the 
corresponding innermost level dish(topic) $\phi_k$.

\subsection{Variations of multi-level nHDP}
Recall that at any given level $l\in\{1,2,\hdots,L-1\}$ of (L+1)-nHDP, HDP distribution of the previous level $H^{l-1}$ becomes the base distribution of the $l^{th}$ level HDP, 
while the group level distributions at the previous level, $G^{l-1}_r, r \in {1,2,\hdots}$, become the atoms for 
the group level distributions at the  $l^{th}$ level, $G^l_{r} , r=1,2, \hdots,$. This leads to multi-level 
admixture modeling where each entity at level $l$ models a distribution over entities at level $l-1$. 
However, one can also consider a variation where entities at a given level are associated with a single 
entity at the previous level leading to a mixture instead of an admixture at this specific level. 
In other words, we replace a given level HDP $H^l$ with a DP to associate a single level-$l$ 
entity with the group at next level. This leads to multi-level model with admixtures at some levels and mixtures 
at other levels. We note that the DP-HDP model(for grouped data) that associates a single entity for each 
document (section $\ref{sec:model}$) is an instance of such a variation.  While these variations open avenues for 
investigating a new set of modeling techniques, we restrict our work to multi-level admixture modeling. Inference in these models should be an extension to that of our admixture model (refer section 5?).

% 
%  Hence from the global menu of dishes at 
% level $L$ governed by $G^L_B$. $\phi_k$. 
% 
% 
% 
%  
% Let $k^l_{rt}$ be the index of the dish chosen in restaurant $r$ for the table $t$ from the menu. 
% The dishes at each level $l$ are drawn from the base distribution of the level $l$ CRF, which is the previous 
% level HDP $H^{L-1}$. Due to the nesting, each dish(atom) in level($l$) corresponds
% to a unique restaurant(group level distribution) at the previous level($l-1$).  
% The dishes in level $0$, $\{\phi_k\}$ are drawn 
% from the base distribution $\bar H$ of the inner most HDP. In the case of the document modeling example, these correspond 
% to topics.
% 
\begin{table}[t]
{
\scriptsize
\centering
\begin{tabular}{|l|c|c|c|}
\hline
Notation & Description of Notation \\
\hline
\hline
$l$ & level index indicated in a superscript \\ \hline
$r$ & Restaurant index \\ \hline
$j$ & Document Index (Used instead of $r$ as index of observed group/restaurant at outermost level $L$)\\ \hline
$i$ & Word (customer) Index within document\\ \hline
$k$ & Dish index in various contexts\\ \hline 
\hline
$K^l$ & Number of dishes in the global menu at level $l$ \\ \hline
$T^l_r$ & Number of tables in restaurant $r$ of level $l$ \\ \hline
$L$ & Index of the outermost level ($l \in \{0, \hdots, L\}$)\\ \hline
\hline
$x_{ji}$ & $i^{th}$ word observed in  $j^{th}$ document \\ \hline
$t^l_{ji}$ &  Table index assigned to word $i$ of document $j$ for level $l$\\ \hline
\hline
%$k^L_{jt}$ & Dish index assigned to table $t$ of document $j$ at outermost level $L$  \\ \hline
$k^l_{rt}$  & Dish index assigned to table $t$ of restaurant $r$ at  level $l$  \\ \hline
$z^l_{j,i}$ &  Dish index at level $l$ assigned to word $i$ of document $j$\\ \hline
$r^l_{j,i}$ & Restaurant index at level $l$ (also level $l+1$ dish index)  for word $i$ of document $j$ \\ \hline
\hline
$\phi^l_k$ & $k^{th}$ dish in the global menu at level $l$ \\ \hline
$\psi^l_{r,t}$ & Dish assigned to $t^{th}$ table in $r^{th}$ restaurant at level $l$ \\ \hline
$\theta^l_{j,i}$ & Dish assigned to $i^{th}$ word of $j^{th}$ document at  level $l$  \\ \hline
\hline
$n^l_{r,t}$ & Number of customers at table $t$ in restaurant $r$ in level $l$ \\ \hline
$m^l_{r,k}$ & Number of tables restaurant $r$ in level $l$ that got assigned dish $k$ \\ \hline
\hline
$\bar H$ & Base distribution of nHDP \\ \hline
$H^l$ &  Base distribution of the HDP at level $l+1$ : $H^l \equiv HDP(\alpha^l,\gamma^l,H^{l-1})$ \\ \hline
$G^l_B$ & Base distribution at level $l$ for group level DP at level $l$\\ \hline
$G^l_r$ & $r^{th}$ Group level distribution at level $l$ \\ \hline
\hline
$\alpha^l$ & Concentration parameter of the group level DP at level $l$ \\ \hline
$\gamma^l$ & Concentration parameter of the base DP at level $l$ \\ \hline
\end{tabular}
\caption{Table describing notation used}
\label{notation}
}
\end{table}

\begin{figure}
\begin{minipage}{0.45\textwidth}
\centering
\includegraphics[trim={6cm 0 0 0},width=90mm]{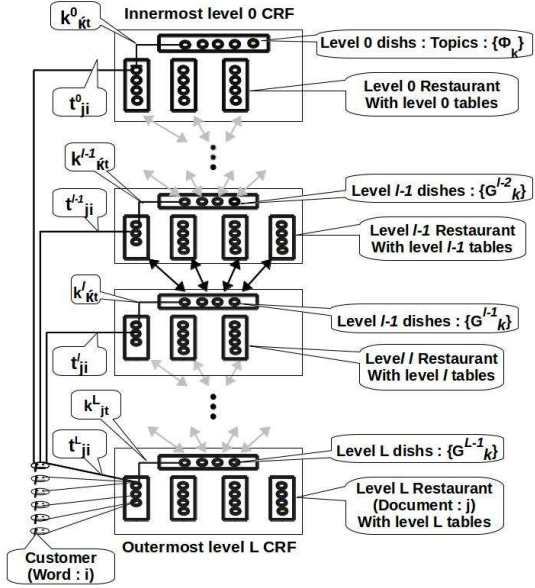}
%\includegraphics[width=0.3\linewidth, height=0.15\textheight]{prob1_6_2}
%\caption{Pictorial representation of Nested Chinese Restaurant Franchise (nCRF)}
%\label{fig:nCRF}
\end{minipage} \hspace{0.1\textwidth}
\begin{minipage}{0.45\textwidth}
\centering
{\scriptsize
\begin{gather*}
 G^0_B \sim DP(\gamma^0,\bar H) \\
 r \in\{1,\hdots\} \tab G^0_{r} \sim DP(\alpha^0, G^0_B) \equiv H^0\\
 \hdots \\[2em]
 G^{l-1}_B \sim DP(\gamma^{l-1},H^{l-2}) \\
 r \in\{1,\hdots\} \tab G^{l-1}_{r} \sim DP(\alpha^{l-1}, G^{l-1}_B) \equiv H^{l-1}\\[2em]
 G^l_B \sim DP(\gamma^l,H^{l-1}) \\
 r \in\{1,\hdots\} \tab G^l_{r} \sim DP(\alpha^l, G^l_B) \equiv H^l\\
 \hdots \\[2em]
 G^L_B \sim DP(\gamma^L,H^{L-1}) \\
 j \in\{1,\hdots,M\} \tab G^L_{j} \sim DP(\alpha^L, G^L_B) 
\end{gather*}
}
%\caption{nHDP }
%\label{fig:ncrf}
\end{minipage}
\caption{Pictorial representation of Nested Chinese Restaurant Franchise (nCRF) on the left and the corresponding nHDP on the right}
\label{fig:nCRF}
\end{figure}

%Want paragraph break here
% Now, observation $x_{ji}$ belonging to the $j^{th}$ outermost level ($L$) group corresponds to the $i^{th}$ customer entering 
% the $j^{th}$ restaurant at level $L$. This customer first selects a table $\theta^L_{ji}$ according to equation 
% \ref{thetaLmulti}. Let $t^L_{ji}$ be the index of the level $L$ table chosen by the customer and $k^L_{jt}=r$ be 
% the index of the dish associated with the specific table from the level $L$ menu. 
% %If this is a new table, a dish $\psi^L_{jt}$ is chosen for this table according to equation \ref{psilmulti}. 
% Now, the customer enters the level $L-1$ restaurant $G^{L-1}_{r}$ and repeats the process by 
% choosing a table $\theta^{L-1}_{ji}$ with index $t^{L-1}_{ji}$. This process continues, where at each level $l$, 
% if customer enters restaurant $k$ (based on choice of dish in previous level), choses the table $\theta^{l}_{ji}$ 
% in this restaurant with index $t^{l}_{ji}=t$, with the corresponding dish $k^l_{kt}=r$, following which
% he proceeds to the restaurant $G^{l-1}_{r}$ in level $l-1$. 
% Finally, at level 0, the customer selects a  dish based on 
% equation  \ref{psi0multi} and the observed data is generated from the corresponding $x_{ji}=F(\phi_k$).
\subsection{ nHDP as Infinite Limit of a Multi-level Finate Mixture Models:}

A Dirichlet process mixture model can be derived as the infinite limit of a finate mixture model as the number of mixture 
componants tends to infinity[Eshwaran and Zaphaeur]. In \cite{HDP}, the authors have shown a similar result where 
a HDP can be constructed as an infinite limit of a collection of finite mixture models. We show a similar result for nHDP
as an infinite limit of multi-level finite mixture models. 

We first define the following collection of finite mixtures. Consider a multi-level setting, 
with $l\in{0, \hdots,L}$ denoting the level, where each level has multiple group level distributions 
$G^l_{r|K^l}$, $r\in\{1,\hdots,K^{l+1}\}$ and a base level distribution $G^l_{0|K^l}$.  
Note that we use the notation $G^l_{r|K^l}$ to denote that the distribution $G^l_{r}$ has a finite number ($K^l$) of atoms.
%Let $K^l$ is the finite number of atoms corresonding to each group level distribution in level l, 
%and where each level $l\in{0, \hdots,L}$ has $K^{l+1}$ groups each corresponding
Further, these group level distributions at each level $l$ form the atoms of the next level $l+1$ defining 
multiple-levels of finite mixtures as follows. 
%to the atoms of the next level $l+1$ and for the outermost group $K^{L+1}=M$ : 
\begin{gather}
 \beta^0 | \gamma^0 \sim Dir (\frac{\gamma^0}{K^0}, \hdots , \frac	{\gamma^0}{K^0})  \tab \text{and} 
 \tab G^0_{0|K^0}= \sum_{k=1}^{K^0} \beta^0_k \delta{\phi_k} \nonumber \\
 \forall r \in \{1,\hdots, K^1\}, \text{ } \pi^0_{r} | \alpha^0 \sim Dir(\alpha^0 \beta^0) \tab \text{and} \tab G^0_{r|K^0}=\sum_{k=1}^{K^0} \pi_{rk} \delta{\phi_k} \nonumber \\
\text{For each level $l \in \{1,\hdots, L-1\}$,} \nonumber \\
\beta^l | \gamma^l \sim Dir (\frac{\gamma^l}{K^l}, \hdots , \frac{\gamma^l}{K^l}) \tab \text{and} \tab G^l_{0|K^l}= \sum_{k=1}^{K^l} \beta^l_{k} \delta_{G^{l-1}_{k}}  \nonumber \\
\forall r \in \{1,\hdots, K^{l+1}\}, \text{ } \pi^l_{r} | \alpha^l \sim Dir(\alpha^l \beta^l)  \tab \text{and} \tab G^l_{r|K^l}=\sum_{k=1}^{K^l} \pi^l_{rk} \delta_{G^{l-1}_{k}} 
%\text{For the outermost level $L$, with M groups, }\\
%\beta^L | \gamma^L \sim Dir (\frac{\gamma^L}{K^L}, \hdots , \frac{\gamma^L}{K^L}) \tab \text{and} \tab G^L_{0|K^L}= \sum_{k=1}^{K^L} \beta^L_{k} \delta_{G^{L-1}_{k}}  \\
%\forall j \in \{1,\hdots, M\}, \text{ } \pi^L_{j} | \alpha^L \sim Dir(\alpha^L \beta^L)  \tab \text{and} \tab G^L_{j|K^L}=\sum_{k=1}^{K^L} \pi^l_{jk} \delta_{G^{L-1}_{k}} 
\end{gather}
\begin{theorem}
\label{theorem1}
For each $l \in \{0,\hdots, L\}$,  with $G^l_{0|K^l}$ and $G^l_{r|K^l}$ defined as above, 
as $K^l \to \infty, \forall l \in \{0,\hdots, L\}$, $G^l_{0|K^l} \to G^l_0$ (with $G^l_0$ as defined in section 
\ref{sec:multilevel}), and $\forall r \in \{1,\hdots, K^{l+1}\}$, $G^l_{r|K^l} \sim DP(\alpha^l, G^l_0)$, tending 
to a draw from an $l$-level nHDP.
\end{theorem}
\begin{proof}
We note that as $K^0 \to \infty$, $G^0_{0|K^0} \to G^0_0$ where the convergence of measures is defined by 
$\int g dG^0_{0|K^0} \xrightarrow[]{D} \int g dG^0_0$ for all real valued functions $g$ measurable with respect to $G^0_0$ as shown by 
[Ishwaran and Zarepour, 2002].  We note that we already have $G^0_{r|K^0} \sim DP(\alpha^0, G^0_{0|K^0})$. 
This follows from the definition of the DP since $G^0_{r|K^0}$ follows equation \ref{eqn:dp} with respect to the 
base measure $ G^0_{0|K^0}$ and the concentration parameter $\alpha^0$.
%for for every finite measurable partitioning $\{A_1, \hdots , A_F\}$ of $\Theta$,  one can show $(G^K_{r}(A_1), \hdots, G^K_{r}(A_F)) \sim Dir(\alpha G^K_0(A_1), \hdots, 
%\alpha G^K_0(A_n)$ from the properties of Dirichlet distribution. 
As $K^0 \to \infty$, we have already established that $G^0_{0|K^0} \to G^0_0$. Hence it follows that 
$G^0_{r|K^0} \sim DP(\alpha^0, G^0_0)$ as $K^0 \to \infty$. 
for each level  $l \in \{1,\hdots, L\}$, having proved this result for all previous levels, assuming 
$K^{l'}\to \infty, \forall l'<l$, we can make a similar argument for level $l$ as that for level 0 
to conclude 
$G^l_{0|K^l} \to G^l_0$, and $\forall r \in \{1,\hdots, K^{l+1}\}$, $G^l_{r|K^L} \sim DP(\alpha^l, G^l_0)$.
This concludes the proof.
\end{proof}

{\bf Alternate construction based on the nCRF}: The following alternate construction based on the nCRF, is another way to show nHDP as an infinite limit of a 
collection of finite mixture models, similar to that in \cite{HDP}, using the table and the dish indices of nCRF from 
the restaurant analogy as follows. 
\begin{gather*} \text{For level $l\in\{1,\hdots,L\}$, } \tab \beta^l | \gamma^l \sim Dir (\frac{\gamma^l}{K^l}, \hdots , \frac{\gamma^l}{K^l}) \\
\forall r \in \{1,\hdots, K^l\},  \pi^l_{r} | \alpha^l \sim Dir(\frac{\alpha^l}{T^l_{r}}, \hdots, \frac{\alpha}{T^l_{r}}) 
\tab k^l_{rt} | \beta^l \sim \beta^l \tab \forall t\in\{1,\hdots,T^l_{r}\}  
\end{gather*}
%t_{ji:Z'_{ji}=r} | \pi_{r} \sim \pi_{r} 
In level L, for each outermost group $j\in\{1,\hdots,M\}$, for each observation, $i\in\{1,\hdots,N_j\}$,\\
\begin{gather}
\label{nCRF_finite}
t^L_{ji} | \pi^L_{j} \sim \pi^L_{j} \tab \text{and} \tab  z^L_{ji}=k^L_{jt^L_{ji}}  \\ \nonumber
% 
% \[ \beta' | \gamma' \sim Dir (\frac{\gamma'}{A}, \hdots , \frac{\gamma'}{A}) \tab  r_{jt'} | \beta' \sim \beta'\]
% \[ \pi'_{j} | \alpha \sim Dir(\frac{\alpha'}{T'_{j}}, \hdots, \frac{\alpha'}{T'_{j}}) \tab  t'_{ji} | \pi'_{j} \sim \pi'_{j}\]
 \forall l\in\{L-1,\hdots,0\} ,~~  \tab  t^l_{ji} | z^{l+1}_{ji}=r, \pi^l_{r} \sim \pi^l_{r} \tab \text{and} \tab  z^l_{ji}=k^l_{z^{l+1}_{ji}t^l_{ji}} \\ \nonumber
 %\[\forall j, \forall i   \]
\tab x_{ji} | z^0_{ji}, \phi \sim \phi_{z^0_{ji}} 
\end{gather}
% 
% \[ \beta | \gamma \sim Dir (\frac{\gamma}{K}, \hdots , \frac{\gamma}{K}) \tab k_{rt} | \beta \sim \beta \]
% \[ \pi_{r} | \alpha \sim Dir(\frac{\alpha}{T_{r}}, \hdots, \frac{\alpha}{T_{r}}) \tab t_{ji:Z'_{ji}=r} | \pi_{r} \sim \pi_{r}  \]
% 
% \[ \beta' | \gamma' \sim Dir (\frac{\gamma'}{A}, \hdots , \frac{\gamma'}{A}) \tab  r_{jt'} | \beta' \sim \beta'\]
% \[ \pi'_{j} | \alpha \sim Dir(\frac{\alpha'}{T'_{j}}, \hdots, \frac{\alpha'}{T'_{j}}) \tab  t'_{ji} | \pi'_{j} \sim \pi'_{j}\]
% \[\forall j, \forall i \tab  Z_{ji}=k_{r_{jt'_{ji}}t_{ji}} \]
%  \[\forall j, \forall i \tab  x_{ji} | Z_{ji}, \phi \sim \phi_{Z_{ji}} \]
% 
\begin{theorem}
 \label{theorem2}
For each $l \in \{0,\hdots, L\}$,  
As $K^l \to \infty, \text{ and } T^l_{r} \to \infty, \forall r\in \{1,\hdots, K^{l+1}\}$,
the generative process described above in equation \ref{nCRF_finite} is equivalent to the nHDP.
\end{theorem}
\begin{proof}At each level $l$,  As $K^l \to \infty, \text{ and } T^l_{r} \to \infty, \forall r\in \{1,\hdots, K^{l+1}\}$,
%As $K \to \infty , T_{r} \to \infty, A \to \infty $ and $T'_{j} \to \infty$, 
the predictive distribution of the draw from each Dirichlet distribution above tends to a CRP %\cite{CRP} 
and hence draws 
of $z^0_{ji}$ in the above construction in equation \ref{nCRF_finite} are the same as that from nCRF described in the previous section. 
Hence the multi-level finite mixture model in \ref{nCRF_finite} tends to nHDP in the infinite limit.
\end{proof}

In the case of L=1, with a single level of nesting, we once again note the similarities between the two-level nHDP and 
the the Author Topic Model(ATM) \cite{AT}. 
%A 1-nHDP finite mixture model from Eqn. \ref{nCRF_finite} is a generalization of ATM. 
%that models 
%documents as a uniform distribution over a finite set of authors  and the authors as a distribution over 
%a finite set of topics. 
With $L=1$ and referring to the index of the outer most group with $j$ instead of $r$, 
$\{G^1_j\}$ parallel the distribution over authors in each document (uniformly distributed in ATM) while $\{G^0_r\}$
parallel the authors' distribution over topics. We note that the finite version of two--level nHDP additionally models the 
base distributions $\{G^1_B\}$ for the global popularity of authors and $\{G^0_B\}$ for the global popularity of topics
leading to a generalization of ATM.  
%Hence we have constructed the nHDP as an infinite limit of a collection of finate mixture models.

%\subsection{ Stick breaking characterization of nHDP}

\section{Inference}
\label{sec:multi-inference}
We use Gibbs sampling for approximate inference as exact inference is intractable for this problem. 
The conditional distributions from the nCRF scheme lend themselves to an inference algorithm, 
where we sample at every level 
$l \in \{0, \hdots, L\}$, the table assignments {$t^l_{ji}$} for customers, and dish assignments {$k^l_{rt}$} for tables 
 where $1 \leq t \leq T^l_{r}$ and restaurant identifier ${r} \in \{1, \hdots, K^{l+1}\}$ 
%$r=k^{l+1}_{j,i}$ 
(Recall $r$ is a restaurant identifier at level $l$ and the number of restaurants in level $l$ is same 
as the number of dishes in level $l+1$). Note that for the outermost level $L$, $K^{L+1}=M$, i.e the number of restaurants 
is the number of observed groups (or documents) in the outermost level. 
%This could correspond to the number of documents in a document modeling setting.
 
The conditional posterior for Gibbs sampling for these variables can be derived from the nCRF conditionals.
However, in such an approach, unlike the inference for a single level HDP, a naive approach of sampling all the 
above indices is intractable leading to an exponential complexity at each level due 
to the tight coupling between the variables. In this section, we first briefly describe 
such an nCRF inference technique (Scheme $1$) by sampling all the variables involved in the nCRF formulation to illustrate 
the computational intractability that arises due to the exponential complexity of this algorithm. 
Following this, we describe in more detail an alternate scheme (Scheme $2$) based on the 
direct sampling technique of HDP that overcomes this problem, that we use for experiments in 
section \ref{experiments} for entity-topic analysis.

In the appendix \ref{singledocInf}, we also discuss scheme $1$ in more detail for a special case, the two--level nHDP using 
which we experimentally demonstrate the difference in complexity between the two schemes.

\subsection{Inference Scheme 1: nCRF Inference}
\label{sec: ncrf}
In the basic nCRF scheme, the latent variables to be sampled as a part of the Gibbs sampling procedure are the assignment of 
tables to each customer $i$ belong to the observed group $j$ and dishes to tables at different levels. 
Hence, we wish to sample at every level,  {$t^l_{ji}$}
and {$k^l_{rt}$} where $r \in \{1, \hdots, K^{l+1}\}$ is a restaurant index at level $l$ also corresponding to 
a dish in level $l+1$ and $t \in \{1, \hdots, T^l_{r}\}$ is a table index in restaurant $r$. 
We start by sampling variables in level $0$, the deepest level, proceeding to variables in level $L$. 
We attempt to illustrate in this section, how the complexity of 
sampling increases, reaching exponential complexity, as we go from sampling variables in level $0$ through level $L$.

The following minor additions to notation are introduced for convenience during inference. 
We denote the set of all observed data as 
${\bf x}=\{x_{ji} : \forall j, i\}$.We denote the set of all customers going to table 
t of restaurant $r$ in level l as ${\bf x^l_{rt}}=\{ x_{ji} : z^{l+1}_{ji}=r, t^l_{ji}=t\}$. Further, a set with a subscript 
starting with a hyphen(-) indicates the set of all elements except the index following the hyphen. 

% Let $\textbf{x} = (x_{ji} :$ all $j,i),$ 
% ${\bf x_{-ji}} = (x_{j'i'} : j' \neq j , i' \neq i),$
% $\textbf{m} = (m^l_{\bar kk} :$ all $\bar k,k,l),$  
% $\textbf{z} = (z^l_{ji} :$ all $j,i,l),$ 
% ${\bf z^l_{-ji}} = (z^{l'}_{j'i'} : j' \neq j , i' \neq i, l' \neq l),$
% and $\beta^l_{new}= \sum_{k=K^l+1}^{\infty} \beta^l_k$

We start with sampling of level $0$ dish assignments to tables, conditioned on values of table and dish assignments at all 
other levels. Hence, we sample  $k^0_{rt}$ as follows, for $r \in \{1, \hdots, K^1\}$, $t \in \{ 1, \hdots, T^0_{r}\}$ by
integrating out $G^0_B$ (using Eqn. \ref{psi0multi})
\begin{align}
\label{level0:ncrf}
p(k^0_{rt} = \bar k| {\bf k^0_{-r,t}}, {\bf k^{-0}}, {\bf t}, {\bf x_{..}}) \propto \\ \nonumber
\begin{cases}
\frac{m^0_{\bar k}}{\gamma^0 + T^0_.}   p({\bf x^0_{rt} } | k^0_{rt} = \bar k, {\bf k^0_{-r,t}}, {\bf k^{-0}}, {\bf t}) \hfill \text{where } \bar k\in\{1, \hdots, K^1\}\\
\frac{\gamma^0}{\gamma^0 + T^0_.}  p({\bf x^0_{rt}}  | k^0_{rt} = K^0+1, {\bf k^0_{-r,t}}, {\bf k^{-0}}, {\bf t})  \hfill \text{  new level 0 dish}
 \end{cases}
\end{align}
The first term is obtained from the conditional probability of the CRP for choosing level $0$ dishes. 
We note that the likelihood terms  $p({\bf x^0_{rt}}  | k^0_{rt} = \bar k, {\bf k^0_{-r,t}}, {\bf k^{-0}}, {\bf t})$  and  
$p({\bf x^0_{rt}}  | k^0_{rt} = K^0+1, {\bf k^0_{-r,t}}, {\bf k^{-0}}, {\bf t})$  arise from the probability of 
all observed data or customers that go to table $t$ of restaurant $r$ at level $0$ that are affected by the assignment 
$k^0_{rt} = \bar k$. These terms can be simplified by integrating out the appropriate $\phi$ variables corresponding to the topic 
multinomials. (A detailed evaluation for these terms is shown in appendix \ref{nCRFlikelihood} for a special case of this 
inference algorithm for ungrouped data).
We further note that this update is similar to that in the direct sampling scheme for a single level HDP in [\cite{HDP}].
% 
% Next, the update for $t^0{ji}$ is obtained as follows for $\forall i,j$
% \begin{equation}
% p(t^0_{j,i} | t^0_{-ji}, k, X) \propto 
% \begin{Cases}
% \frac{m^0_{\bar k}}{\gamma^0 + T^0_.}  f_{\bar k} (x^0_{kt}) \hfill \text{where } \bar k=1, \hdots, K^0\\
% \frac{\gamma^0}{\gamma^0 + T^0_.}  f_{\bar k_{new}} (x^0_{kt})  \hfill \text{  new level 0 dish} 
% \end{Cases}
%\end{equation}

For the next level, we sample the update for dish assignment to tables belonging to level $1$ restaurants, $k^1_{rt}$, for each 
$r \in \{1, \hdots, K^2\}$, $t \in \{ 1, \hdots, T^1_{r} \}$. 
Let ${\bf S}^1_{rt}={\{t^0_{ji}: t^1_{ji} = t, z^2_{ji} = r \}}$, the set of level $0$ table 
assignments corresponding to all  customers ${j,i}$ who have been assigned the table $t$ in level $1$ restaurant $r$,
%(Also corresponding to a level 2 dish assignment of $r$). 
% \begin{align}
% p(k^1_{rt} = \bar k| k^1_{r,t}, k^{-1}, t^{-0}, X) \propto \\ \nonumber
% \begin{cases}
% \frac{m^1_{\bar k}}{\gamma^1 + T^1_.}  p(x^1_{rt}  | k^1_{rt} = \bar k, k^1_{-r,t}, k^{-1}, t^{-0} , t^0_{\{ji: t^1_{ji} \neq t|| Z^2_{ji} \neq k \}}) \hfill \text{where } \bar k=1, \hdots, K^1\\
% \frac{\gamma^0}{\gamma^0 + T^0_.}  p(x^1_{rt} | k^1_{rt} = K^1+1, k^1_{-r,t}, k^{-1}, t^{-0}, t^0_{\{ji: t^1_{ji} \neq t || Z^2_{ji} \neq k \}})  \hfill \text{  new level 1 dish}
%  \end{cases}
% \end{align}
\begin{align}
p(k^1_{rt} = \bar k| {\bf k^1_{-r,t}}, {\bf k^{-1}}, {\bf t^{-0}}, {\bf x}) \propto \\ \nonumber
\begin{cases}
\frac{m^1_{\bar k}}{\gamma^1 + T^1_.}  p({\bf x^1_{rt}}  | k^1_{rt} = \bar k, k^1_{-r,t}, k^{-1}, \{-{\bf S}^1_{rt} \}) \hfill \text{where } \bar k=1, \hdots, K^1\\
\frac{\gamma^0}{\gamma^0 + T^0_.}  p({\bf x^1_{rt}} | k^1_{rt} = K^1+1, {\bf k^1_{-r,t}}, {\bf k^{-1}}, \{-{\bf S}^1_{rt} \})  \hfill \text{  new level 1 dish}
 \end{cases}
\end{align}
We note that the likelihood terms are conditioned on all table assignments 
except those in the set ${\bf S}^1_{rt}$ since changing the level 1 dish assignment $k^1_{rt}$ of 
the table in restaurant $r$ changes the level 0 restaurant that the customer enters, due to which table assignments
in set ${\bf S}^1_{rt}$ are not known.

Hence, evaluating the likelihood term requires marginalizing over all possible 
assignments for latent variables ${\bf S}^1_{rt}$. We note that each of these variables can take a value 
between $1, \hdots, T^0_{\bar k}+1$. This leads to $(T^0_{\bar k}+1)^{|{\bf S}^1_{rt}|}$ operations to simplify the likelihood term 
leading to an exponential complexity for evaluating the update $k^1_{kt}$ rendering this inference technique intractable.

We see that similarly, for a general level $l$, sampling $k^l_{rt}$ for $k= \in \{1, \hdots, K^{l+1}\}$, 
$t \in \{ 1, \hdots, T^0_{r}\} $ 
requires the marginalization over the following set ${\bf S}^l_{rt}$ of all table assignments in all previous levels 
$\hat l < l$ for the customers sitting at the particular table $t$ in level $l$ restaurant $r$ :
$${\bf S}^l_{rt}={\{t^{\hat l}_{ji}: \hat l<l, t^l_{ji} = t, z^{l+1}_{ji} = r \}}$$
We see that the cardinality of this set increases exponentially with increasing $l$ due to which this technique is intractable for a general $l$, 
the only exception being $l=0$ for a single level HDP where this technique is tractable as in equation \ref{level0:ncrf}.

\subsection{Inference Scheme 2: Direct Sampling Scheme}
\label{sec: direct}
To work around the exponential complexity encountered in the previous section, we adopt a technique similar to the 
direct sampling scheme in [\cite{HDP}] where the variables   {$t^l_{ji}$}, and {$k^l_{rt}, \forall l,j,i,r,t$}
are not explicitly sampled. Instead the variables $G^l_0$ are explicitly sampled for all levels $l$ as opposed to being 
integrated out, by sampling the stick breaking weights $\beta^l$  respectively. Further, 
we directly sample {$z^l_{ji}$}, the \emph{dish assignment} at level $l$ for each customer(word) $i$, in each group(document) $j$, 
avoiding explicit assignments of tables to customers and dishes to tables. 
However, in order to sample $\beta^l$, the table information is maintained in the form of the aggregated 
counts in each layer, {$m^l_{r k}$}, the number of tables at level $l$, in restaurant $ r \in \{1, \hdots, K^{l+1}\}$ 
assigned to dish  $k \in \{1, \hdots, K^{l}\}$. (Recall that each restaurant at level $l$ corresponds to a unique dish 
in level $l+1$. Hence, $ r \in \{1, \hdots, K^{l+1}\}$. )
Thus the latent variables that need to be sampled in the Gibbs sampling 
scheme are {$z^l_{ji}$},  $\beta^l$, {$m^l_{r k}$}, $\forall l,i,j, r, k$.

We introduce the following notation for the rest of this section. 
Let $\textbf{x} = (x_{ji} :$ all $j,i),$ 
${\bf x_{-ji}} = (x_{j'i'} : j' \neq j , i' \neq i),$
$\textbf{m} = (m^l_{r k} :$ all $r,k,l),$  
$\textbf{z} = (z^l_{ji} :$ all $j,i,l),$ 
${\bf z^l_{-ji}} = (z^{l'}_{j'i'} : j' \neq j , i' \neq i, l' \neq l),$
and $\beta^l_{new}= \sum_{k=K^l+1}^{\infty} \beta^l_k$
We now provide the sampling updates for dish assignments for customers at each level, starting from level $0$, 
conditioned on all other dish assignments at all levels. 

%\vspace*{0.1in} \noindent
{\bf Sampling {$z^0_{ji}$}: } The conditional distribution for the dish assignment at level $0$,  $z^0_{ji}$, depends on the 
predictive distribution of the dish assignment $z^0_{ji}$, given all other dish assignment to customers at this level 
and all other levels and the emission probabilities of the final observed data $x_{ji}$ with the specific dish assignment. 
This is given by
\begin{equation*}
%\label{inference_zji}
\begin{split}
 p(z^0_{ji} = p |z^1_{ji}=r , {\bf z^0_{-ji}},  {\bf m}, \beta, {\bf x}) 
 \propto p(z^0_{ji} = p| {\bf z^0_{-ji}}, z^1_{ji}=r ) p(x_{ji} | z^0_{ji} = p, {\bf x_{-ji}}) 
 \end{split}
\end{equation*}

To pick dish $p$ at level $0$, conditioned on the dish assignment at level $1$ as $r$, the first term 
$p(z^0_{ji} = p| {\bf z^0_{-ji}}, z^1_{ji}=r )$ can be split into two parts. One for 
 picking any of the existing tables from the level $0$ restaurant $r$ that get mapped to dish $p$ 
 and one from creating a new table in restaurant $r$ and assigning dish $p$ to it. 
In the instance of  choosing a new dish, a new table is always created in restaurant $r$ at level $0$. Hence,
 \begin{equation}
\label{inference_zji_a}
 p(z^0_{ji} = p| z^1_{ji}=r, {\bf z^0_{-ji}} ) \propto 
 \begin{cases}
\frac{n^0_{r.p} + \alpha \beta^0_p}{n_{r,.}+\alpha^0}   \hfill \text{  Existing dish}\\
\frac{\alpha^0 \beta^0_{new}}{n^0_{r.}+\alpha^0}   \hfill \text{  New dish}
 \end{cases}
\end{equation}

The likelihood term $p(x_{ji} | z^0_{ji} = p, {\bf x_{-ji}}) $ is the conditional density of $x_{ji}$ under level $0$ dish(topic) $z^0_{ji}=p$  given all data items except $x_{ji}$. 
Assuming the $0$ level dish is a topic sampled from a $V$ dimensional symmetric Dirichlet prior over the vocabulary 
with parameter $\eta$, i.e $\phi_p \sim Dir(\eta)$, the conditional can be simplified to the following expression, 
by integrating out $\phi_p$.
 \begin{equation*}
%\label{inference_zji_b}
 p(x_{ji} = w | z^0_{ji} = p, {\bf x_{-ji}})  \propto \frac{ n^0_{pw} + \eta}{n^0_p. +  V \eta}
\end{equation*}
where $n^0_{pw}$ is the number of occurrences of level 0 dish(topic) $p$ with word $w$ in the vocabulary.
We note that this step is similar to that in [\cite{HDP}].

%\vspace*{0.1in}\noindent
{\bf For any general level, sampling {$z^l_{ji}$}: } The conditional distribution for the dish assignment at level $l$ 
is computed as  
\begin{equation*}
%\label{inference_z'ji}
\begin{split}
 p(z^l_{ji} = p | {\bf z^l_{-ji}}, z^{l+1}_{ji}=r, z^{l-1}_{ji}=q, {\bf z^l_{-ji}} ,  {\bf m}, \beta,  {\bf x})  \\
 \propto p(z^l_{ji} = p| {\bf z^l_{-ji}}, z^{l+1}_{ji} = r) p(z^{l-1}_{ji} = q| {\bf z^{l-1}_{-ji}}, z^l_{ji} = p )
 \end{split}
\end{equation*}

 The first term is the predictive distribution of $z^l_{ji}$ given the next level dish assignment $r$ (to specify
 which level $l$ restaurant the customer goes to), while the second term arises from the previous level dish 
 assignment $q$ that depends on the value of $z^l_{ji}$. Again, $p(z^l_{ji} = p| {\bf z^l_{-ji}}, z^{l+1}_{ji} = r)$ can be viewed as consisting of two terms. One from 
 picking an existing table in restaurant $r$ with dish assignment $p$ and one from creating a new table 
 in restaurant $r$ at level $l$ and assigning the dish $p$ to it. 
 Further, creation of a new dish always involves the creation of a new table. Hence, 
 \begin{equation*}
%\label{inference_z'ji_a}
 p(z^l_{ji} = p | {\bf z^l_{-ji}}, z^{l+1}_{ji} = r ) \propto 
 \begin{cases}
\frac{n^l_{rp} + \alpha^l \beta^l_{p}}{n^l_{r.}+\alpha^l} \hfill \text{Existing dish}\\
\frac{ \alpha^l \beta^l_{new}}{n^l_{r.}+\alpha^l} \hfill \text{New dish}\\
 \end{cases}
\end{equation*}

Similarly 
\begin{equation*}
%\label{inference_z'ji_a}
 p(z^{l-1}_{ji} = q | {\bf z^{l-1}_{-ji}}, z^{l}_{ji} = p ) \propto 
 \begin{cases}
\frac{n^{l-1}_{pq} + \alpha^{l-1} \beta^{l-1}_{q}}{n^{l-1}_{p.}+\alpha^{l-1}} \hfill \text{Existing dish}\\
\frac{ \alpha^{l-1} \beta^{l-1}_{new}}{n^{l-1}_{p.}+\alpha^{l-1}} \hfill \text{New dish}\\
 \end{cases}
\end{equation*}

%$p(z_{ji} = z| {\bf z_{-ji}}, z'_{ji} = z' )$ follows from Eqn. \ref{inference_zji_a}.

%\vspace*{0.1in} \noindent
{\bf Sampling $\beta$ : }
At each level $l$, the posterior of $G^l_0$, conditioned on samples observed from it, is also distributed as a DP due 
to Dirichlet-Multinomial conjugacy, and the stick breaking weights of $G^l_0$ can be sampled as follows:
%\begin{equation*}
%\label{inference_beta}
$$(\beta^l_1, \beta^l_2 \hdots \beta^l_{K^l}, \beta_{new} ) \sim Dir(m^l_{.1}, m^l_{.2} \hdots m^l_{.K}, \gamma^l)$$
%\end{equation*}

%\vspace*{0.1in} \noindent 
{\bf Sampling $m$ : }
$m^l_{r k}$ is the number of tables in level $l$ restaurant $r \in \{1, \hdots, K^{l+1}\}$ 
%(also corresponding to a dish in level $l+1$)
 that are assigned to the level $l$ dish $k \in \{1, \hdots, K^{l}\}$.
%where $K^{L+1}=M$ 
%This is equivalent to the the number of new tables generated in restaurant $r$ while simulating a CRF  where
In other words, $m^l_{r k}$ is the number of tables created as  $n^l_{r.k}$ samples are drawn from $G^l_{r}$ in restaurant $r$ that correspond to a particular dish $k$. 
This is the number of partitions generated as samples are drawn from a Dirichlet Process with concentration parameter 
$\alpha^l \beta^l_k$ and are distributed according to a closed form expression [\cite{antoniak:aos74}]. 
However, we adapt an alternate method [\cite{Fox:AOAS2011}] for sampling $m^l_{rk}$  by drawing a 
total of $n^l_{r.k}$ samples with dish k, and incrementing the count $m^l_{r k}$ whenever a new table is created 
in restaurant $r$ with dish assignment k.
% with probability $\alpha \beta_k$, for the topic $k$ 
%with probability $\alpha' \beta'_{\bar k}$, for entity $\bar k$.

%\vspace*{0.1in}\noindent 
{\bf Sampling Concentration parameters: }
We place a vague gamma prior on the concentration parameters $\alpha^l$, $\gamma^l$ $\forall l$ with hyper parameters 
$\alpha_a , \alpha_b , \gamma_a, \gamma_b$ respectively. 
We use Gibbs sampling scheme for sampling the concentration parameters using the technique outlined in [\cite{HDP}].

\section{Experimental evaluation of inference complexity}
\label{sec:expinf}
%In this section we briefly discuss the inference complexity of the two proposed inference schemes, the nCRF scheme and the 
%direct sampling scheme. 

The nCRF scheme (Scheme 1) is computationally more expensive than the direct sampling scheme.
Scheme 1, as described in section \ref{sec: ncrf}, runs to exponential complexity even for the 2 level 
nHDP. Hence, we introduced the direct sampling scheme in section \ref{sec: direct} to outline a tractable inference algorithm. 
In this section, we illustrate this through some examples.

First we perform experiments with the single level nHDP, to compare the inference (training time) 
with both these schemes. The results of this experiment is shown in figure \ref{L1:compare_runtime}.
We also compare the perplexity obtained on held out test data with both these schemes for the single level nHDP. 
The perplexity results on the NIPS dataset with 20 percent of the documents held out is shown in 
table \ref{L1perplexity}. We note that while the nCRF scheme (scheme 1) performs better in terms of perplexity,
the direct sampling scheme is faster. This difference in complexity increases exponentially as we add more levels 
to the nHDP.

To better illustrate the difference in computational complexity between the two schemes, in this section, 
we compare the runtime of these two algorithms for a special case of our 2-level nHDP model where there is a single 
restaurant in the outer most level. In this special setting, at the outer most level, the HDP can be replaced by a simple DP since 
there is no sharing of atoms required between restaurants.
We discuss both the naive nCRF and the direct sampling inference algorithm for this setting in detail in appendix \ref{singledocInf}.
We perform experiments on a 100 document subset of the NIPS dataset to compare runtime in this 
special setting. The results are shown in figure \ref{compare_runtime}.
We note that the direct sampling technique is order of magnitude faster with respect to runtime complexity as 
illustrated in the figures \ref{compare_runtime}.
%and \ref{compare_likelihood}.
\begin{figure}[!Htbp]
\begin{subfigure}{.48\textwidth}
\centering
\includegraphics[height=2.5in,width=2.5in]{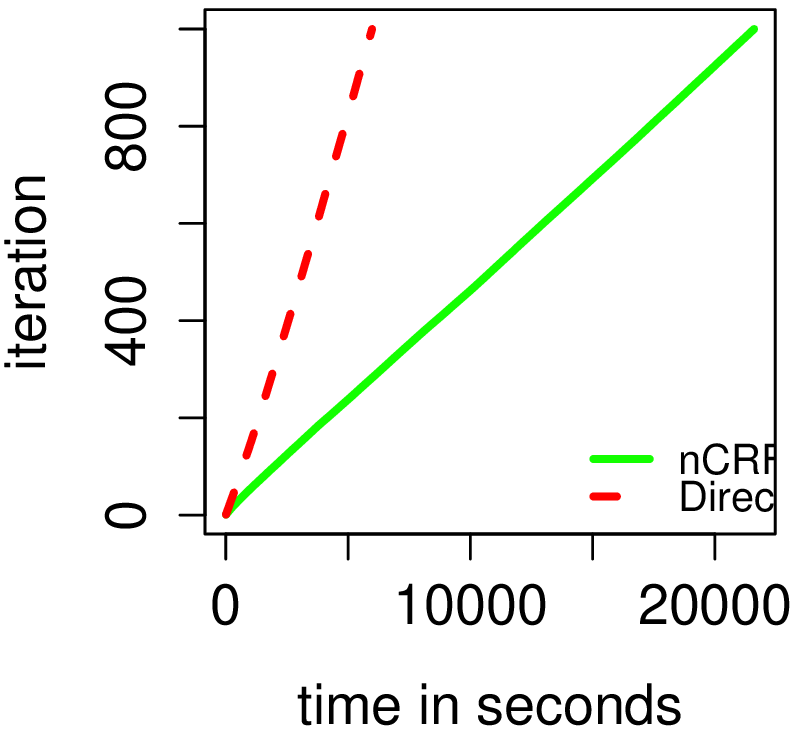}
\caption{\small \sl  Single level nHDP (HDP) : We see that direct sampling is faster than the nCRF scheme. However the difference is 
not as pronounced in the case with more levels}  
\label{L1:compare_runtime}
\end{subfigure}
\begin{subfigure}{.04\textwidth}
\end{subfigure}
\begin{subfigure}{.48\textwidth}
%\end{figure*}  
%\begin{figure*}[!Htbp]
\centering
\includegraphics[height=2.5in,width=2.5in]{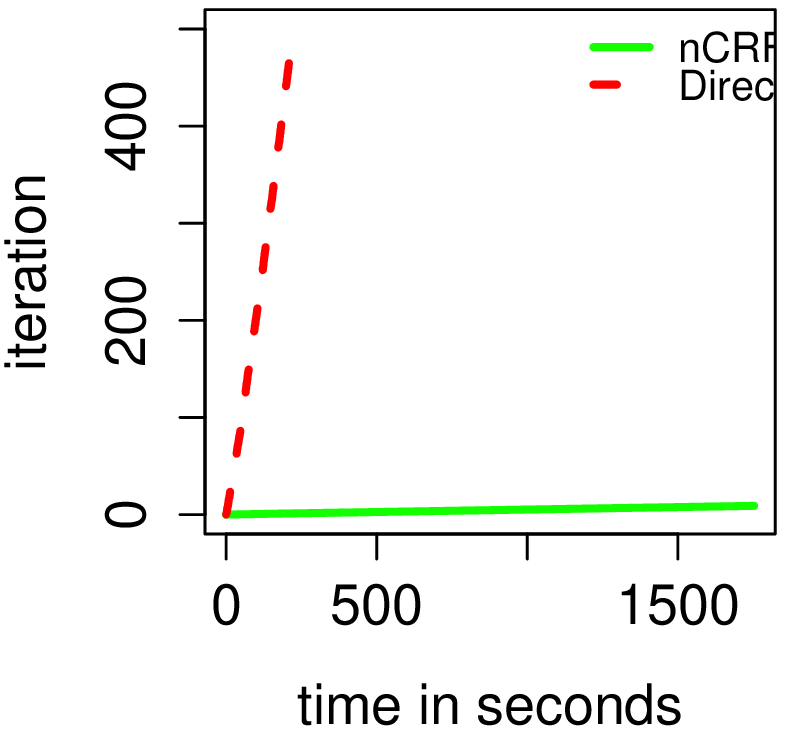}
\caption{\small \sl two-level nHDP for ungrouped data : We see that direct sampling runs for several iterations 
before the nCRF technique completes the first few Gibbs sampling iterations}  
\label{compare_runtime}
\end{subfigure}
\caption{Comparison of Runtime : Direct Sampling vs nCRF}
\end{figure}  
\begin{table}[t]
\centering
\begin{tabular}{|l|c|c|}
\hline
Model & Direct Sampling & nCRF \\
\hline
\hline
Perplexity & 2230 & 1937 \\
\hline
\end{tabular}
\caption{Perplexity comparison for single level nHDP for NIPS dataset with different inference schemes}\label{tab:L1co}
\label{L1perplexity}
\end{table}

\section{Non-Parametric entity-topic Modeling : Experimental Analysis}
\label{experiments}
In this section, we experimentally evaluate the proposed nHDP model in the context of non-parametric entity-topic modeling, 
with a two-level nHDP, for the task of modeling author entities who have collaboratively written research papers, 
 and compare its performance against available baselines. Specifically, we evaluate two different aspects: 
(1) how well the model is able to learn from the training samples and fit held-out data in terms of perplexity, 
(a) first, when all the authors are observed in training and test documents, and 
(b) secondly, when some of the authors are unobserved in training and test documents,
(c) finally, when all authors are unobserved, to understand the effect of multi-level HDP in comparison with a single level 
level HDP on perplexity.
(2) how accurately the model discovers hidden authors, who are not mentioned at all in the corpus. 

We consider the following models for the experiments:
(i) The author-topic model({\bf ATM}) (Eqn. \ref{eqn:atm}) where the number of topics is pre-specified, and all authors are observed for all documents.  
%We use code provided by the authors \footnote{url} and set the number of topics to ... 
This is used as a baseline for (1a) above.
% We compare of our model with \textit{Author-Topic Model} (ATM) \cite{AT}, a parametric model with number of topics and authors fixed and the set of authors known for each document, 
(ii) The Hierarchical Dirichlet Process ({\bf HDP}) (Eqn. \ref{eqn:hdp}) using the direct assignment inference scheme for fair comparison. 
We use our own implementation for this. 
Recall that the HDP is infers the number of topics, and does not use author information.%, and is used as a baseline for both (1a) and (1b) above.
%and \textit{Hierarchical Dirichlet Process} (HDP) \cite{HDP}, a non-parametric topic model, which does not characterize the authors/entities of the documents. 
%We consider the following variants of nHDP to analyze the effect of the amount of author information observed by our model.
%\begin{itemize}
%\item 
(iii) nHDP with completely observed entities ({\bf nHDP-co}), which assumes complete entity information to be available for all documents, but is learns topics in a nonparametric fashion. 
This can be imagined as an improvement over ATM where the number of topics does not need to be specified.
%This is used for (1a) above. 
%\textbf{Completely Observed (CO):} This variant uses  the complete information of the authors who have contributed to each of the documents and serves as a non-parametric extension (w.r.t. topics) of Author Topic Model (ATM).
%\item 
(iv)  nHDP with partially observed entities ({\bf nHDP-po}), which makes use of available entity information, but admits the possibility of entities being hidden globally from the corpus, or locally from individual documents. 
%This is used for (1b) above.
%\textbf{Partially Observed (PO):} In this variant, for each document we make partial observations on its author set and allow the model to identify the missing authors.
%\item 
(v) nHDP with no observed entities ({\bf nHDP-no}), which does not make use of any entity information and assumes all entities to be globally hidden in the corpus.
%\textbf{Completely Unobserved (CU):} As the name suggests, this variant uses no knowledge about the authors and is completely unsupervised in terms of topics as well as authors.
%\end{itemize}
For task (1a) above, the applicable models are the ATM, HDP (which ignores the entity information) and nHDP-co.
For task (1b) and (1c), the ATM does not work.
We evaluate HDP, and nHDP-po / nHDP-no.
It is important to point out that there are no available baselines in terms of entity-topic analysis for task (2) above 
when some or all of the authors are unobserved.

%\subsection{Datasets Used}
We use the following publicly available publication datasets for our experimental analysis.
%\begin{enumerate}
%\item 
The \textbf{NIPS } dataset\footnote{http://www.arbylon.net/resources.html}  is a collection of papers from Neural Information Processing Systems (NIPS) conference proceedings (volume 0-12). 
This collection contains 1,740 documents contributed by a total of 2,037 authors, with total 2,301,375 word tokens resulting in a vocabulary of 13,649 words.
%\item 
A subset of the \textbf{DBLP Abstracts}  dataset\footnote{http://www.cs.uiuc.edu/~hbdeng/data/kdd2011.htm}  containing 
12,000 documents by 15,252 authors collected from 20 conferences records on the Digital Bibliography and Library Project (DBLP). 
Each document is represented as a bag of words present in abstract and title of the corresponding paper, resulting in a vocabulary size of 11,771 words.
%\item \textbf{NSF-Awards dataset }consists of a subset of abstracts from NSF Research Awards for basic research from 1990 to 2003. This collection has 16,405 abstracts by 9,989 investigators and a total of 18,674 unique vocabulary terms.
%\end{enumerate}
%\\
%Creation of test train splits.\\
%Hiding authors.

%\subsection{Baselines}
%We compare our model against the following two approaches which  we believe have more similarity compared to other existing models.
%\begin{enumerate}
%\item \textbf{Author-Topic Model: }
%\end{enumerate}

%\subsection{Results and Discussion}
%\subsubsection{Fitting of Data: }

\vspace{0.05in}
\noindent
{\bf 1. Generalization Ability: }
We now come to our first experiment, where we evaluate the ability of the models, whose parameters are learnt from a training set, to predict words in new unseen documents in a held-out test set.
%For this analysis we divide our corpus into a training set, using which we learn the model and and a test set, on which we evaluate the performance of the learnt model in terms of perplexity. 
We evaluate performance of a model $M$ on a test collection $D$ using the standard notion of perplexity [\cite{LDA}]: $exp(-\sum_{d\in D} p(w_d)|M)$.

In experiment (1a), all authors are observed in training and test documents.
To favor the ATM, which cannot handle new authors in test document, we create test-train splits ensuring that each author in the test collection occurs in at least one training document. 
%First, we compare the performance of our model with the existing parametric Author-Topic model. Being parametric in nature, ATM does not work  for the case when test set has authors which not present in the training set. Therefore, we split our datasets such that each of the authors in the dataset have authored at least one training document. 

%\begin{figure}
%\caption{Comparison of Training Likelihood of nHDP-co, HDP and ATM(for best K=10)}
% \includegraphics[width=90mm]{exp1_nips.eps}
%\label{fig:ll}
%\end{figure}

\begin{table}[t]
\centering
%\begin{tabular}{|l|c|}
%\hline
%Model& Perplexity\\
%\hline
%\hline
%ATM & 2783\\
%\hline
%HDP & 1775\\
%\hline
%nHDP\-co & 1247 \\
%\hline
%\end{tabular}
\begin{tabular}{|l|c|c|c|}
\hline
Model & ATM & HDP & nHDP-co \\
\hline
\hline
Perplexity & 2783 & 1775 & 1247 \\
\hline
\end{tabular}
\caption{Perplexity of ATM, HDP and nHDP-co for NIPS}\label{tab:co}
\end{table}

%Improvement in loglikelihood for different models over training iterations is plotted in Fig. \ref{fig:ll} and 
Perplexity results are shown in Table \ref{tab:co}.
%Note that ATM has an advantage over our CU model as ATM observes the author information of each document, while our completely unobserved model does not any author-related information. 
Recall that HDP and nHDP finds the best number of topics, while for ATM we have recorded its best performance across different value of $K$.
The results show that while knowledge of authors is useful, the ability of non-parametric topic models to infer the number of topics clearly leads to better generalization ability.

Next, in experiment (1b), we first create training-test distributions with reasonable author overlap by letting each author vote with probability $0.7$ whether to send a document to train or test, and majority decision is taken for each document. 
Next, authors are partially hidden from both the test and the train documents as following.
%Next, we investigate if modeling entities makes better prediction of words for held-out documents than HDP and also demonstrate that  when entities are partially or completely unobserved, the ability to model hidden entities allows nHDP to fit better than HDP. 
We iterate over the global list of authors and remove this author from all training and test documents with probability $p_g$.
We then iterate over each training and test document, and remove each remaining author of that document with probability $p_l$.
We experiment with different values of $p_g$ and $p_l$ to simulate different extents of missing information on authors.
$p_g=1$ and $p_l=1$ corresponds to (1c), the case where authors are completely unobserved. This setting enables us to compare
the two-level nHDP, with completely unobserved dishes at each level, with a HDP, to understand the relative merit of 
multi-level modeling over a single level in terms of perplexity.

%For these experiments, we create a train-test split and then simulate the absence of author information by hiding the authors for the documents with different probability values, thus creating (partially observed) splits with varying amount of author information hidden. Note that these removals are done such that an author might be hidden only for a subset of documents it has authored (local removal) or could be hidden completely for all the documents (global removal). 

\begin{table}[t]
\centering
%\begin{tabular}{|l|c|c|c|c|c}
%\hline
%Model&$p_g$& $p_l$&NIPS&DBLP\\
%%    & &      &NIPS dataset& DBLP Dataset\\
%\hline
%\hline
%HDP&0&0&2572&1027 \\
%\hline
%nHDP-no&0&0&1882&997 \\
%\hline
%nHDP-po&0.4&0.4&1434&935 \\
%\hline
%nHDP-po&0.6&0.6&1266&869 \\
%\hline
%nHDP-po&0.8&0.8&1109&676\\
%\hline
%nHDP-co&1&1&987&394\\
%\hline
%\end{tabular}
\begin{tabular}{|l|c|c|c|c|c|c|}
\hline
Model & HDP & nHDP-no & nHDP-po & nHDP-po & nHDP-po & nHDP-co \\
\hline
\hline
$p_g$,$p_l$ & 1,1 & 1,1 & 0.6,0.6 & 0.4,0.4 & 0.2,0.2 & 0,0 \\
Perplexity NIPS & 2572 & 1882 & 1434 & 1266 & 1109 & 987\\
Perplexity DBLP & 1027 & 997 & 935 & 869 & 676 & 394\\
\hline
\end{tabular}
\caption{Perplexity for HDP and nHDP with varying percentage of hidden authors %during training on the NIPS and DBLP dataset. $p_g$=1,$p_l$=1 corresponds to the nHDP-co model and $p_g$=0,$p_l$=0 corresponds to the nHDP-no model
}\label{tab:po}
\end{table}

The results are shown in Table \ref{tab:po}.
We can see that more information available about the authors, the ability to fit held-out data improves.
%({\bf TBD} Can we report the number of new authors proposed by nHDP-po for different levels of hidden authors?)
More interestingly, even when no / very little author information is available, just the assumption about the existence of 
a discrete set of authors, i.e introducing an additional layer of HDP, leads to better generalization ability, corroborating
the need for multi-level modeling, as can be seen from the relative performance of HDP and nHDP-no.

%\subsubsection{Author Discovery:}
\vspace{0.05in}
\noindent
{\bf 2. Discovering Missing Authors: }
Beyond data fitting, the most significant ability of our model is to discover entities which are relevant for documents in the corpus, but are never mentioned.
We perform a case study with the top $6$ most prolific authors in NIPS, by removing them completely from the corpus, and then checking the ability of the model to discover them in a completely unsupervised fashion. 
%We conduct two  different experiments: one with only that author removed from the entire dataset and another with $x\%$ of his collaborating authors are additionally removed from the dataset.
While it is possible to define as a classification problem the task of detecting of {\it locally missing} authors in individual documents when the author is observed in other documents, we reiterate that there is no existing baseline when an author is {\it globally hidden}.

We evaluate the accuracy of discovering hidden author as follows.
For each hidden author $h\in\{1\ldots H\}$, we create a $m$-dimensional vector $c_h$, where $m$ is the corpus size, with $c_h[j]$ indicating his authorship in the $j^{th}$ document.
We explored two possibilities for this `true' indicator vector: (a) binary indicators using the gold-standard author names for documents, and (b) the number of words written by that author in the document according to nHDP with completely observed authors (nHDP-co).
Similarly, we create an $m$-dimensional vector for each new author $n\in\{1\ldots N\}$ discovered by the nHDP-po, with $c_n[j]$ indicating his contribution (no. of authored words) in the $j^{th}$ document.
We now check how well the vectors $\{c_n\}$ correspond to the `true' vectors $\{c_h\}$. 
This is done by defining two variables $C_n$ and $C_h$, taking values $1\ldots H$ and $1\ldots N$ respectively, and defining a joint distribution over them as $P(h,n)=\frac{1}{Z}\text{sim}(c_h,c_n)$, where $Z$ is a normalization constant.
For $\text{sim}(c_h,c_n)$, we use cosine similarity between normalized versions of $c_h$ and $c_n$.
Mutual information $I(C_h,C_n)=\sum_{h,n}p(h,n)\log\frac{p(h,n)}{p(h)p(n)}$ measures the information that $C_h$ and $C_y$ share.
We used its normalized variant $NMI(C_h,C_n)=\frac{I(C_h,C_n)}{|H(C_h)+H(C_n)|/2}$ ($H(X)$ indicating entropy of $X$) which takes values between $0$ and $1$, higher values indicating more shared information.

%For the hidden author $h$, let $D_h$ be the set of documents which he has authored.
%We create an $m$-dimensional vector $c_h$, where $m=|D_h|$ and $c_h[i]$ is the fraction of total words with author label $h$ according to nHDP-co in the $i^{th}$ document in $D_h$.
%We similarly create $m$-dim vectors $c_j$ for the contribution in those same documents for each new author discovered by nHDP-po, when only author $h$ is globally removed.
%In case of perfect discovery of author $h$ by nHDP-po, there would be exactly one new author with $c_{j*}=c_h$, and any other new author should have $c_j=\bf{0}$.
%So we measure the accuracy of discovery using ...

%The results are recorded in Table \ref{tab:disc}.
First, we note that the best NMI achievable for this task, by replacing the true vectors $\{c_h\}$ for the discovered vectors $\{c_n\}$, is $0.86$ for case (a) and $0.98$ for case (b) above. 
In comparison,  using nHDP-po, we achieve NMI scores of $0.59$ for case (a) and $0.72$ for case (b). 
This indicates that the actual author distributions that the model discovers not only help in fitting the data, but also have reasonable correspondence with the true hidden authors.
We believe that this is a promising initial step in addressing this difficult problem.
\section{Conclusions}
\label{sec:conclusion}

In this paper, we have proposed the the nested Hierarchical Dirichlet Process as a prior for multi-level admixture modeling. 
We have also addressed the problem of entity-topic analysis from document corpora, where the set of document 
entities are either completely or partially hidden through the two level nHDP,
which consists of two levels of Hierarchical Dirichlet Processes, where one is the base 
distribution of the other. We explore inference algorithms for nHDP and using a direct sampling scheme for inference, 
we have shown that the nHDP is able to generalize better than existing models under varying available knowledge about 
authors in research publications, and is additionally able to discover completely hidden authors in the corpus.

\bibliography{nhdp}  % sigproc.bib is the name of the Bibliography in this case
\section{Appendix}
\subsection{Two-level Inference with Ungrouped data at Outermost Level }
\label{singledocInf}
In this section, we describe the collapsed Gibbs sampling inference for the setting with ungrouped data at 
the outermost level in the entity-topic application for document modeling. This is a special case of the 
two level nHDP model with a DP in the outer level instead of a HDP. While we use notation similar to the nHDP 
inference described in section \ref{sec:multi-inference}, the observed data is indexed by a single index $i$ (the index $j$ vanishes
since there is no demarcation into groups i.e documents at the outermost level).
The nCRF representation for this setting involves assigning a dish(entity) $z^1_i$ with index $k^1_i$ to every customer 
$i$ based on $G^1_B$, the global distribution over entities based on which the customer enters an inner level restaurant $r^0_i$. 
At this restaurant the customer picks a table $\hat t=t^0_i$ which is assigned a corresponding dish $z^0_i$ with index 
$k^1_{\hat t}$ that corresponds to a topic. The observed data is generated based on the topic assignment thus attained.

We now describe the two inference schemes described in section \ref{sec:multi-inference} for the two level nHDP for 
entity topic modeling for this special case of 
ungrouped data. Note that in section \ref{sec:expinf}, an experimental comparison of both these schemes is shown.

\underline{\textit{Scheme 1: Naive nCRF based Sampling for entity-topic modeling of ungrouped data:}}

The latent variables to be sampled include $t^0_i, k^1_i$ for each observation $i=1, \hdots, M$ and $k^0_{rt}$ for 
restaurant $r=1, \hdots, K^1$ and table $t=1, \hdots, T^0_{r}$. Sampling $k^0_{rt}$ is similar to that in the full 
nCRF inference procedure described in the previous section and is not described here.

The update for selecting the level 0 table $t^0_i$ for each customer can be obtained as follows by integrating out the appropriate
$G^0_{r}$.
\begin{equation}
p(t^0_i = \hat t, z^1_i=r | t^0_{-i}, X, k) \propto
\begin{cases}
 \frac{n^0_{r \hat t}}{n^0_{r .} + \alpha^0} p (x_i | t^0_i = \hat t)  \hfill \text{Existing $\hat t = 1, \hdots, T^0_{r}$}\\
 \frac{\alpha^0}{n^0_{r .} + \alpha^0} p (x_i | t^0_i = T^0_{r}+1)   \hfill \text{New table $\hat t= T^0_{r}$}\\
 \end{cases}
\end{equation}
Where $ p (x_i | t^0_i = T^0_{r}+1)$ can be evaluated as follows considering the different level 0 dishes that can be assigned
to the new level 0 table. 
 \[p (x_i | t^0_i = T^0_{z^1_i}+1)=\sum_{r=1}^{K^0} \frac{m^0_{r}}{m^0_{.} + \gamma^0} p(x_i | z^0_i = r)
 + \frac{\gamma^0}{m^0_{.} + \gamma^0}  p(x_i | z^0_i = K^0 +1 ) \]
The overall cost of this update step is O($M$ $T^0_{max}$ $K^0$).
 
The update for $k^1_i$ can be obtained by integrating out $G^1_B$ as follows.
\begin{equation}
p(k^1_i = r | k^1_{-i}, X, t^0_{-i}) \propto
\begin{cases}
 \frac{m^1_{r}}{m^1_{.} + \gamma^1} p (x_i | k^1_i = r)  \hfill \text{Existing $r = 1, \hdots, K^1$}\\
 \frac{\gamma^1}{m^1_{.} + \gamma^1}  p (x_i | k^1_i =  K^1+1)  \hfill \text{New table $r= K^1+1$}\\
 \end{cases}
\end{equation}
Changing the value of $k^1_i$, invalidates the existing assignment to $t^0_i$, Hence
evaluating $p (x_i | k^1_i = r)$ requires summing over possible values of $t^0_i$ as follows 
$$p (x_i | k^1_i = r)=\sum_{\hat t=1}^{T^0_{r}} \frac{n^0_{r \hat t}}{n^0_{r \hat t}+\alpha^0} p(x_i |t^0_i=\hat t,k^1_i=r)
          +\frac{\alpha^0}{n^0_{r \hat t}+\alpha^0} p(x_i |t^0_i=T^0_{r},k^1_i=r)$$
In turn, $p(x_i |t^0_i=T^0_{r},k^1_i=r)$, corresponds to the case where a new level 0 table is created and requires 
summing over all potential value of level 1 dish assignments for this table. 
Hence, $p(x_i |t^0_i=T^0_{r},k^1_i=r$ and similarly  $p (x_i | k^1_i =  K^1+1)$ that involve the creation of 
a new level 1 table, can be evaluated as follows, 
$$p(x_i |t^0_i=T^0_{r},k^1_i=r) = 
\sum_{r=1}^{K^0} \frac{m^0_{r}}{m^0_.+\gamma^0} p(x_i | z^0_i=r)
+ \frac{\gamma^0}{m^0_.+\gamma^0} p(x_i | z^0_i=K^0+1)$$

\underline{\textit{Scheme 2: Direct Sampling for entity-topic modeling of Ungrouped data}}\\
The latent variables involved in the direct sampling scheme are $z^0_i$ and $z^1_i$ for each observation $i=1, \hdots, M$.
Sampling $z^0_i$ is similar to that for the case of full nHDP direct sampling
\begin{equation*}
%\label{inference_zji}
\begin{split}
 p(z^0_{i} = p |z^1_{i}=r , {\bf z^0_{-i}},  {\bf m}, \beta, {\bf x}) 
 \propto p(z^0_{i} = p| {\bf z^0_{-i}}, z^1_{i}=r ) p(x_{i} | z^0_{ji} = p, {\bf x_{-i}}) 
 \end{split}
\end{equation*}
The first term can be expanded to the following, similar to that in the full nHDP with $n^0_{r.p}$ defined in section \ref{sec: direct}.
\begin{equation}
\label{eqn:single_inference_zi_a}
 p(z^0_{i} = p| z^1_{i}=r, {\bf z^0_{-i}} ) \propto 
 \begin{cases}
\frac{n^0_{r.p} + \alpha \beta^0_p}{n_{r,.}+\alpha^0}   \hfill \text{  Existing dish}\\
\frac{\alpha^0 \beta^0_{new}}{n^0_{r.}+\alpha^0}   \hfill \text{  New dish}
 \end{cases}
\end{equation}
The second term can be simplified similar to section \ref{sec: direct} by integrating out the $\phi$ multinomials corresponding
to the dishes in the inner most level.

The update for $z^1_i$ can be similarly obtained as 
\begin{equation*}
%\label{inference_z'ji}
\begin{split}
 p(z^1_{i} = p | z^1_{-i}, z^{0}_{i}=q, {\bf z^1_{-i}} ,  {\bf m}, \beta,  {\bf x})  \\
 \propto p(z^1_{i} = p| {\bf z^l_{-i}}) p(z^{0}_{i} = q| {\bf z^{0}_{-i}}, z^1_{i} = p )
 \end{split}
\end{equation*}
The first term is the conditional of a simple CRP while the second term simplifies as 
\begin{equation*}
%\label{inference_z'ji_a}
 p(z^{0}_{i} = q | {\bf z^{0}_{-i}}, z^{1}_{i} = p ) \propto 
 \begin{cases}
\frac{n^{0}_{pq} + \alpha^{0} \beta^{0}_{q}}{n^{0}_{p.}+\alpha^{0}} \hfill \text{Existing dish}\\
\frac{ \alpha^{0} \beta^{0}_{new}}{n^{0}_{p.}+\alpha^{0}} \hfill \text{New dish}\\
 \end{cases}
\end{equation*}

%\subsection{Evaluation of likelihood terms in the naive nCRF inference scheme}
%\label{nCRFlikelihood}

\end{document}